\documentclass[final,12pt]{main} 

\title[Online learning in dynamically changing environments]{Online Learning in Dynamically Changing Environments}
\usepackage{times}
\usepackage{algpseudocode}
\usepackage{algorithm}
\usepackage{amssymb}
\usepackage{makecell}
\usepackage{tablefootnote}

\coltauthor{%
 \Name{Changlong Wu} \Email{wuchangl@hawaii.edu}\\
 \Name{Ananth Grama} \Email{ayg@cs.purdue.edu}\\
 \Name{Wojciech Szpankowski} \Email{szpan@purdue.edu}\\
 \addr CSoI, Purdue University
}

\newcommand{\x}{\textbf{x}}

\newcommand{\z}{\textbf{z}}

\newcommand{\erm}{\mathsf{ERM}}
\newcommand{\vch}{\mathsf{VC}(\mathcal{H})}

\begin{document}

\maketitle

\begin{abstract}%
  We study the problem of online learning and online regret minimization when samples are drawn from a general unknown \emph{non-stationary} process. We introduce the concept of a \emph{dynamic changing process} with cost $K$, where the \emph{conditional} marginals of the process can vary arbitrarily, but that the number of different conditional marginals is bounded by $K$ over $T$ rounds. For such processes we prove a tight (upto $\sqrt{\log T}$ factor) bound $O(\sqrt{KT\cdot\vch\log T})$ for the \emph{expected worst case} regret of any finite VC-dimensional class $\mathcal{H}$ under absolute loss (i.e., the expected miss-classification loss). We then improve this bound for general mixable losses, by establishing a tight (up to $\log^3 T$ factor) regret bound $O(K\cdot\vch\log^3 T)$. We extend these results to general \emph{smooth adversary} processes with \emph{unknown} reference measure by showing a sub-linear regret bound for $1$-dimensional threshold functions under a general bounded convex loss. Our results can be viewed as a first step towards regret analysis with non-stationary samples in the \emph{distribution blind} (universal) regime.
  This also brings a new viewpoint that shifts the study of complexity of the hypothesis classes to the study of the complexity of processes generating data.
\end{abstract}


\section{Introduction}
\label{sec:intro}
We study the problem of online learning and online regret minimization with statistically generated samples, when compared with a broad class of experts. Unlike the classical setting in online learning where samples are assumed to be generated adversarially, we consider the case in which samples are drawn from a general stochastic process (possibly \emph{non-stationary}). Formally, we consider the following game between two parties, named Nature and predictor, played over $T$ rounds. In the beginning, Nature selects some distribution $\pmb{\nu}^T$ over $\mathcal{X}^T$ (i.e., a random process) and samples $\x^T\sim \pmb{\nu}^T$ where $\mathbf{x}^T=(\mathbf{x}_1, \ldots, \mathbf{x}_T)$. At each time step $t\le T$, Nature reveals $\x_t$  to the predictor, who makes a prediction $\hat{y}_t=\phi_t(\x^t,y^{t-1})$ potentially using the history $\x^t=(\x_1,\cdots,\x_t)$ and $y^{t-1}=(y_1,\cdots,y_{t-1})$ that are observed thus far. Nature then reveals the true label $y_t$ after the prediction and the predictor incurs a loss $\ell(\hat{y}_t,y_t)$ for some predefined \emph{convex} loss function $\ell:\hat{\mathcal{Y}}\times \mathcal{Y}\rightarrow [0,\infty)$. We are interested in the following \emph{expected worst case} regret:
\begin{equation}
    \label{eq-rtilde}
\tilde{r}_T(\mathcal{H},\mathsf{P})=\inf_{\phi^T}\sup_{\pmb{\nu}^T\in \mathsf{P}}\mathbb{E}_{\x^T\sim\pmb{\nu}^T}\left[\sup_{y^T}\sum_{t=1}^T\ell(\phi_t(\x^t,y^{t-1}),y_t)-\inf_{h\in \mathcal{H}}\sum_{t=1}^T\ell(h(\x_t),y_t)\right],
\end{equation}
where $\mathcal{H}$ is a class of functions $\mathcal{X}\rightarrow \hat{\mathcal{Y}}$, $\mathsf{P}$ is a general class of random processes over $\mathcal{X}^T$, and $\phi^T$ runs over all possible (deterministic) prediction rules.

Online learning has been mostly studied in literature under the assumption that samples are presented adversarially~\citep{ben2009agnostic,rakhlin2010online,rakhlin2015martingale}. However, the generality of the adversary assumption often comes with the cost that only very restricted classes can be handled with sub-linear regret predictors. For instance, for binary valued classes and absolute loss, one has to assume that the class $\mathcal{H}$ has finite Littlestone dimension, which already rules out some of the simple classes of interest, e.g., 1-dimensional threshold functions. Recent results \citep{haghtalab2020smoothed,haghtalab2022smoothed,block2022smoothed} have demonstrated that these restrictions can be substantially relaxed (i.e., from finite Littlestone dimension to finite VC-dimension) by considering a more optimistic process for generating samples, i.e., smooth adversary samples. Formally, one assumes that there exists some \emph{known} reference measure $\mu$ over the instance space $\mathcal{X}$, such that at each time step $t$, an adversary selects some distribution $\nu_t$ that  is \emph{$\sigma$-smooth} w.r.t. $\mu$ for generating the next sample. Here, \emph{smoothness} is understood as follows: for any event $A\subset \mathcal{X}$, we have $\nu_t(A)\le \frac{1}{\sigma}\mu(A)$.~\cite{haghtalab2022smoothed} showed that one can achieve sublinear regret bounds under absolute loss with an $\log(1/\sigma)$ dependency on regret for any finite VC-dimensional class if the instances are generated by a smooth adversary process with known $\mu$. This was further generalized~\citep{block2022smoothed} to the real valued case with finite scale-sensitive VC-dimension (i.e., fat-shattering number) and with computationally efficient predictors (using an ERM oracle).

This paper follows a similar path by considering relevant intermediate scenarios between the full adversary case and full $i.i.d.$ case. Instead of assuming some \emph{known} reference measure $\mu$ that determines the generating process as in~\citep{haghtalab2020smoothed,haghtalab2022smoothed,block2022smoothed}, we consider a \emph{universal} scenario where we do not assume any knowledge about the process generating the instances; instead, we require that the change in the processes are constrained in certain ways. Our goal is to understand the restrictions under which one is able to obtain sub-linear regret bounds for finite VC-dimensional classes. To achieve this, we consider the following broad scenario:

\paragraph{Universal smooth process:}
Let $\mu_1,\cdots,\mu_K$ be probability measures over $\mathcal{X}$, $\mathcal{S}^{\sigma}(\mu_k)$ be the set of all $\sigma$-smooth distributions over $\mathcal{X}$ with reference measure $\mu_k$, and $\nu_t(X_t\mid X^{t-1})$ be the distribution of $X_t$ conditioning on $X^{t-1}$. Then
a random process $X^T$ over $\mathcal{X}^T$ with \emph{joint} distribution $\pmb{\nu}^T$ is said to be a $(K,\sigma)$-smooth process, if:
\begin{equation}
\label{eq:universalsmooth}
    \mathrm{Pr}\left[\exists \mu_1,\cdots,\mu_K,~s.t.~\forall t\in [T],~\nu_t(X_t\mid X^{t-1})\in \bigcup_{k\in [K]}\mathcal{S}^{\sigma}(\mu_k)\right]=1.
\end{equation}
We denote by $\mathsf{U}_K^{\sigma}$ the class of \emph{all} $(K,\sigma)$-smooth processes. Let $\mathsf{S}^{\sigma}(\mu_1,\cdots,\mu_K)$ be the class of all $\sigma$-smooth random process with reference measures $\mu_1,\cdots,\mu_K$; i.e., for any $\pmb{\nu}^T\in \mathsf{S}^{\sigma}(\mu_1,\cdots,\mu_K)$, we have for all $t\in [T]$,~$\nu_t(X_t\mid X^{t-1})\in \bigcup_{k\in [K]}\mathcal{S}^{\sigma}(\mu_k)$ almost surely. Note that the processes considered in~\citep{haghtalab2020smoothed,haghtalab2022smoothed,block2022smoothed} is simply $\mathsf{S}^{\sigma}(\mu)$ for a single known reference measure $\mu$. We also write $\Tilde{\mathsf{U}}_{K}^{\sigma}=\bigcup_{\mu_1,\cdots,\mu_K}\mathsf{S}^{\sigma}(\mu_1,\cdots,\mu_K)$, where $\mu_1,\cdots,\mu_K$ run over all $K$-tuples of distributions over $\mathcal{X}$. It is easy to show (see Propositions~\ref{prop_equiv} and~\ref{prop_transit}) that $\mathsf{U}_K^{\sigma}\subset \mathsf{U}_1^{\sigma/K}$ and $\Tilde{\mathsf{U}}_{K}^{\sigma}\subset \Tilde{\mathsf{U}}_{1}^{\sigma/K}$. Moreover, $\mathsf{S}^{\sigma}(\mu_1,\cdots,\mu_K)\subsetneq\Tilde{\mathsf{U}}_{K}^{\sigma}\subsetneq \mathsf{U}_K^{\sigma}$, where the inclusion is  \emph{strict}.

\vspace{-0.1in}
\subsection{Results and Techniques} 
We emphasize that the class $\mathsf{U}_K^{\sigma}$ is a very broad class of processes and includes many interesting and natural settings. We do not intend to provide a full characterization for such a broad class in this paper. Instead, we study the following two sub-categories of $\mathsf{U}_K^{\sigma}$, which are of significant interest, 
with results summarized in
Table~\ref{tab:table1}:

\paragraph{Dynamic changing process with cost $K$:} A process $X^T$ is said to be a \emph{dynamic} changing process of cost $K$, if $|\{\nu_1(X_1),\nu_2(X_2\mid X_1),\cdots,\nu_T(X_T\mid X^{T-1})\}|\le K$ almost surely; i.e., the \emph{conditional} marginals (for any sample path) can have at most $K$ different distributions over $T$ rounds, while the possible distributions are themselves \emph{unconstrained}. It is easy to observe that this class is simply $\mathsf{U}_K^1$. Our first main result achieves the optimal \emph{expected worst case} regret defined in (\ref{eq-rtilde}) for such processes with finite VC-dimensional class under absolute loss upto a $\sqrt{\log T}$ factor for a wide range of $\vch, T$, and $K$.
\begin{theorem}[Theorem~\ref{thm_sqrt_bound}]
For a convex and bounded loss function $\ell$ with finite VC of
$\mathcal{H}\subset \{0,1\}^{\mathcal{X}}$:
    $\tilde{r}_T(\mathcal{H},\mathsf{U}_K^1)\le O(\sqrt{KT\cdot \vch\log T})$
    provided $K^3\cdot \vch \le O(T^{1-\epsilon}/\log T)$ with constant $\epsilon>0$.
    Furthermore, for $d,K\ge 1$ with $Kd\le O(T/\log T)$, we have
    $\sup_{\mathcal{H},\vch\le d}\tilde{r}_T(\mathcal{H},\mathsf{U}_K^1)\ge \Omega(\sqrt{KdT})$
    under absolute loss.
\end{theorem}

Our main \emph{algorithmic} technique to establish Theorem~\ref{thm_sqrt_bound} is an \emph{adaptive} epoch-EWA approach presented in Algorithm~\ref{alg:1}, where we maintain a finite set of experts at each epoch and update the epochs \emph{adaptively} according to the sample we observed, unlike the conventional approach that defines the epochs independent of the samples, such as~\citep{lazaric2009hybrid}.

Our second main result is the following regret bound under mixable losses:
\begin{theorem}[Corollary~\ref{cor:mixable}]
\label{thm_intro2}
    Suppose $\ell$ is a bounded mixable loss (or logarithmic loss), $\mathcal{H}\subset \{0,1\}^{\mathcal{X}}$ is a class of finite VC-dimension. Then $\Tilde{r}_T(\mathcal{H},\mathsf{U}_K^1)\le O(K\cdot \vch\log^3T\cdot \Delta),$
    where $\Delta=\log(\vch\log(KT))$. Moreover, for any $Kd\log d\le O(T)$, there exists a class $\mathcal{H}$ with $\vch\ge d$ such that $\Tilde{r}_T(\mathcal{H},\mathsf{U}_K^1)\ge \Omega(Kd)$ under logarithmic loss.
\end{theorem}

The main technique for establishing Theorem~\ref{thm_intro2} is the \emph{stochastic sequential covering}, introduced in the recent paper~\citep{wu2022expected} (see also \citep{wu2022precise}), together with a \emph{perturbation} technique for establishing a \emph{realizable} cumulative error bound for ERM rule under $\mathsf{U}_K^1$, which may be of independent interest.

\paragraph{The class $\tilde{\mathsf{U}}_1^{\sigma}$ with $\sigma<1$:} Our next main result is a reduction from the class $\tilde{\mathsf{U}}_1^{\sigma}$ to the class of \emph{adversary $K$-selection} processes using a similar coupling argument as in~\citep{haghtalab2022smoothed,block2022smoothed}. We say a random process $X^T$ is adversary $K$-selection process if there exists a coupling $V^{KT}$ of $X^T$ such that for all $t\in [T]$ we have $X_t\in \{V_{K(t-1)+1},\cdots,V_{Kt}\}$ \emph{almost surely} and $V^{KT}$ is an $i.i.d.$ process. Using this reduction and stochastic sequential covering, we establish in Corollary~\ref{cor_threhold_smooth} the regrets for $1$-dimensional threshold functions under $\tilde{\mathsf{U}}_1^{\sigma}$ of order 
$O\left(\sqrt[4]{{(T^3/\sigma)\log^2(T/\sigma)}}\right)$ for absolute loss and $O\left(\sqrt{({T}/{\sigma})\log^2(T/\sigma)}\right)$ for mixable losses.

\paragraph{Summary of main contributions.} We formulate the online learning problem
with changing environment in which the underlying data distribution is
unknown (universality) and \emph{non-stationary}. We also analyze the expected worst case
regret for universal processes generated by smooth adversaries with \emph{unknown} reference measures.
Our formulation shifts the focus from the complexity of hypothesis classes  to the complexity            
of processes generating samples.
On the algorithmic side,
we design a new \emph{adaptive} epoch-EWA algorithm that is of independent interest and we
expect it will find other applications. On the methodology side, we design a novel
\emph{stochastic sequential} covering approach to obtain upper bounds on regret, which is applicable for \emph{general} random processes.
For matching lower bounds,
we introduce a novel technique based on the concept of Littlestone \emph{forests}.
We stress that for general universal smooth processes we
restrict our analysis to the threshold functions as the first step
towards better understanding of this complex problem.
While the threshold function may seem simple from the classical learning perspective, we emphasize 
that the analysis is nontrivial due to complex structure of the universal smooth adversary processes.

\renewcommand{\arraystretch}{2}
\begin{table}[h!]
  \begin{center}
    \caption{Summary of Results}
    \label{tab:table1}
    \begin{tabular}{l|c|c|c}
    \hline
       & \makecell{$\mathsf{S}^{\sigma}(\mu)$\\ (VC class)} &\makecell{$\mathsf{U}_K^1$\\(VC class)} & \makecell{$\Tilde{\mathsf{U}}_1^{\sigma}$\\ (Threshold functions)}\\
      \hline
      \makecell{Absolute loss}
      & \makecell{$O\left(\sqrt{\mathsf{VC}\cdot T\log\frac{T}{\sigma}}\right)$\\ $\Omega\left(\sqrt{\mathsf{VC}\cdot T\log\frac{1}{\mathsf{VC}\cdot\sigma}}\right)$\\\citep{haghtalab2022smoothed}} &\makecell{ $O(\sqrt{\mathsf{VC}\cdot KT\log T})$ \\ $\Omega(\sqrt{\mathsf{VC}\cdot KT})$\\ (Theorem~\ref{thm_sqrt_bound})} &\makecell{$O\left(\sqrt{\frac{T^{3/2}\log(T/\sigma)}{\sigma^{1/2}}}\right)$\\ (Corollary~\ref{cor_threhold_smooth})}\\
      \hline
      Mixable loss & \makecell{ $O\left(\mathsf{VC}\cdot \log\frac{T}{\sigma}\right)$\\ $\Omega\left(\mathsf{VC}\cdot \log\left(\frac{T}{\mathsf{VC}}\vee\frac{1}{\mathsf{VC}\cdot \sigma}\right)\right)$\\ (Corollary~\ref{cor1}) } & \makecell{$O(\mathsf{VC}\cdot K\log^3 T)$ \\ $\Omega(\mathsf{VC}\cdot (K\vee \log \frac{T}{\mathsf{VC}}))$ \\ (Corollary~\ref{cor:mixable}) } & \makecell{ $O\left(\sqrt{\frac{T}{\sigma}\log^2\frac{T}{\sigma}}\right)$\\ (Corollary~\ref{cor_threhold_smooth}) }\\
      \hline
    \end{tabular}
  \end{center}
 \quad \footnotesize{$^*$ The bounds hold for certain ranges of the parameters given in the referenced theorems and $a\vee b=\max\{a,b\}$.}
\end{table}

\vspace{-0.18in}
\subsection{Related work}
Online learning from randomized samples was first investigated in~\citep{haussler1994predicting}, where the authors considered the case in which features $\x^T$ are sampled from some unknown $i.i.d.$ source and $y^T$ is \emph{realized} by some function $h\in \mathcal{H}$. It is shown in~\citep{haussler1994predicting} that one can achieve a $O\left({\vch}/{T}\right)$ expected \emph{error rate} in such a scenario using the so called \emph{1-inclusion graph} algorithm. This result was latter strengthen and extended in~\citep{schuurmans1997characterizing,antos1998strong,wu2021non,bousquet2021theory}. However, all of these results assumed that the samples must be realizable by some function in $\mathcal{H}$. \cite{lazaric2009hybrid} considered an alternate scenario in which features $\x^T$ are  $i.i.d.$, but the labels $y^T$ are adversarial. It is shown in~\citep{lazaric2009hybrid} that one can achieve a $O(\sqrt{T\cdot \vch\log T})$ \emph{regret} under absolute loss if $\mathcal{H}$ is a binary valued class of finite VC-dimension. This scenario was extended in~\citep{wu2022expected} to general distributions for features $\x^T$ and general losses for which the authors also introduced the notion of the \emph{expected worst case} regret. Despite the general formulation in~\citep{wu2022expected}, only $i.i.d.$ (i.e., exchangeable) distributions were analyzed. Others~\citep{rakhlin2011online1,haghtalab2020smoothed,haghtalab2022smoothed,block2022smoothed} have studied more sophisticated processes, namely the \emph{smooth adversary} process\footnote{Note that the regrets analyzed in these papers can be rephrased as the expected worst case regret.}. However, it was assumed that the reference measure of the smooth adversary samples must be known in advance\footnote{For unknown distributions, we need substantially different techniques, as demonstrated in this paper.}. We note also that~\cite{bilodeau2020relaxing} consider similar intermediate scenarios but with finite expert classes. Online learning with general distributions is also discussed in~\citep{hanneke2021learning}.

There has been a lot of work on online learning problems with adversarial samples; please see~\citep{lugosi-book, hazan2016introduction} for excellent discussions of this topic. We note that the term "changing environments" has also been used in the online learning literature with different meanings. \cite{blum2007external} and~\cite{hazan2009efficient} studied \emph{changing environments} interpreted as minimizing the regret by comparing to some \emph{changing compactors} (instead of a static compactor); however, the samples are still assumed to be adversary. 
In this paper we focus primarily on how the changing sampling process affects regret when the compactor is still assumed to be static and coming from a \emph{large} (possibly non-parametric) class $\mathcal{H}$.

\vspace{-0.1in}
\section{Preliminaries}
\label{sec_formulation}
Let $\mathcal{X}$ be a feature (instance) space, $\hat{\mathcal{Y}}=[0,1]$ be the prediction space, and $\mathcal{Y}=\{0,1\}$ be the true label space. 
We denote by $\mathcal{H}\subset \hat{\mathcal{Y}}^{\mathcal{X}}$ 
a class of functions $\mathcal{X}\rightarrow \hat{\mathcal{Y}}$, 
which is also referred to as a hypothesis or experts class. For any time horizon $T$, 
we consider a class $\mathsf{P}$ of distributions over $\mathcal{X}^T$. We are interested in the \emph{expected worst case} minimax regret $\Tilde{r}_T(\mathcal{H},\mathsf{P})$ as defined in~(\ref{eq-rtilde}) under a general convex loss $\ell$. This includes, for instance, the absolute loss $\ell(\hat{y},y)=|\hat{y}-y|$ (which can be interpreted as $\mathbb{E}_{\hat{b}\sim \text{Bernoulli}(\hat{y})}[1\{\hat{b}\not=y\}]$) and the logarithmic loss $\ell(\hat{y},y)=-y\log \hat{y}-(1-y)\log(1-\hat{y})$. Using minimax inequality, it is easy to observe that
$$\tilde{r}_T(\mathcal{H},\mathsf{P})\ge \sup_{\pmb{\xi}^{2T}}\inf_{\phi^T}
\mathbb{E}_{(\x^T,y^T)\sim \pmb{\xi}^{2T}}\left[\sum_{t=1}^T\ell(\phi_t(\x^t,y^{t-1}),y_t)-
\inf_{h\in \mathcal{H}}\sum_{t=1}^T\ell(h(\x_t),y_t)\right],$$
where $\pmb{\xi}^{2T}$ is a joint distribution over $\mathcal{X}^T\times \mathcal{Y}^T$ such that the marginal distribution of $\pmb{\xi}^{2T}$ restricted on $\mathcal{X}^T$ is in $\mathsf{P}$. We will use such a relation to derive \emph{lower bounds} for $\tilde{r}_T$.

In this paper, we assume that $\mathcal{H}\subset \{0,1\}^{\mathcal{X}}$ is binary valued~\footnote{We assume $\mathcal{H}$ to be binary valued for the clarity of presentation. However, our results also hold for \emph{embedding} of $\mathcal{H}$ into real valued functions such as in~\citep{bhatt2021sequential}, see Appendix~\ref{sec:real}.} and has finite VC-dimension. We  specifically study here how the structure of the \emph{distribution class} $\mathsf{P}$ impacts  expected worst case regret. This is unlike most of the results in learning theory literature that focus on the impact of the structure of $\mathcal{H}$ on regret. We now provide several examples of $\mathsf{P}$ that demonstrate how previously considered setups in the literature fit  into our framework.
\begin{example}
    If $\mathsf{P}$ is the class of all singleton distributions over $\mathcal{X}^T$, our setup recovers the adversary setting, as in~\citep{rakhlin2010online}. If $\mathsf{P}$ is the class of all $i.i.d.$ processes over $\mathcal{X}^T$, our setup recovers those of~\citep{lazaric2009hybrid}.
\end{example}

\begin{example}[The smooth adversary setting]
\label{exm_smooth}
    The smooth adversary setting is an intermediate setting between the full adversary and the $i.i.d.$ case. In this setting, one assumes that there is some (known) underlying reference measure $\mu$ over $\mathcal{X}$, such that at each time step $t$ an adversary selects some $\sigma$-smooth distribution $\nu_t$ w.r.t. $\mu$ that generates sample $\x_t$. Formally, we say a distribution $\nu$ is $\sigma$-smooth (with $\sigma\le 1$) w.r.t. to $\mu$ if $\nu$ is absolutely continuous w.r.t. $\mu$ and has density $v(\x)=\frac{\text{d}\nu}{\text{d}\mu}$ such that $\mu\left(\left\{\x:v(\x)\le {1}/{\sigma}\right\}\right)=1.$ We denote  by $\mathcal{S}^{\sigma}(\mu)$ the class of all $\sigma$-smooth distributions w.r.t. $\mu$. We say a process $\pmb{\nu}^T$ over $\mathcal{X}^T$ is $\sigma$-smooth w.r.t. $\mu$ if for all $t\le T$ the conditional distribution $\nu_t(X_t\mid X^{t-1})$ of $X_t$ conditioning on $X^{t-1}$ is in $\mathcal{S}^{\sigma}(\mu)$ almost surely. We write $\mathsf{S}^{\sigma}(\mu)$ 
    for the class of all such random processes. Using a standard \emph{skolemization} argument~\citep{rakhlin2010online}, the minimax regret for any class $\mathcal{H}$ w.r.t. smooth adversaries, as in~\citep{haghtalab2020smoothed,haghtalab2022smoothed,block2022smoothed}, is simply
    $\tilde{r}_T(\mathcal{H},\mathsf{S}^{\sigma}(\mu)).$ We refer to Appendix~\ref{sec:known_dist} for a self-contained discussion of regret analysis w.r.t. $\mathsf{S}^{\sigma}(\mu)$ with extensions to broader losses.
\end{example}

A crucial restriction of the smooth adversary setting of  Example~\ref{exm_smooth} is that the reference distribution $\mu$ must be \emph{known} and fixed. A more interesting and realistic scenario is when the reference measure itself is allowed to change. More generally, one may have \emph{no} knowledge about the reference measures. Our main focus of this paper is the \emph{universal smooth process} $\mathsf{U}_K^{\sigma}$, as defined in Equation~(\ref{eq:universalsmooth}); in particular, the sub-classes $\mathsf{U}_K^1$ and $\Tilde{\mathsf{U}}_1^{\sigma}$ (see Section~\ref{sec:intro} for formal definitions). 

The following propositions provide a useful reduction from multiple reference measures to \emph{one} reference measure, i.e., $\mathsf{U}_K^{\sigma}\subset \mathsf{U}_1^{\sigma/K}$, $\Tilde{\mathsf{U}}_K^{\sigma}\subset \Tilde{\mathsf{U}}_1^{\sigma/K}$ and $\mathsf{S}^{\sigma}(\mu_1,\cdots,\mu_K)\subset \mathsf{S}^{\sigma/K}(\Bar{\mu})$.

\begin{proposition}
\label{prop_equiv}
    Let $\mu_1,\cdots,\mu_K$ be $K$ \emph{arbitrary} distributions over the same domain $\mathcal{X}$. Then for all $k\in [K]$ the measure $\mu_k$ is $1/K$-smooth w.r.t. $\bar{\mu}$, where $\bar{\mu}=\frac{1}{K}\sum_{k=1}^K\mu_k.$
\end{proposition}
\begin{proof}
    Note that $\bar{\mu}$ is interpreted as follows: for any measurable event $A\subset \mathcal{X}$, we have $\bar{\mu}(A)=\frac{1}{K}\sum_{t=1}^K\mu_k(A).$ It is easy to verify that $\bar{\mu}$ is a probability measure over $\mathcal{X}$. We now show that, for all $k\in [K]$, $\mu_k$ is $1/K$-smooth w.r.t. $\bar{\mu}$. To see this, we observe that $\mu_k$ is absolutely continuous w.r.t. $\bar{\mu}$. By Radon–Nikodym theorem, there is a density $u_k(\x)=\frac{\text{d}\mu_k}{\text{d}\bar{\mu}}$ of $\mu_k$ w.r.t. $\bar{\mu}$. Let $\mathcal{E}_k=\{\x:u_k(\x)> K\}$. We have $\mu_k(\mathcal{E}_k)/K>\bar{\mu}(\mathcal{E}_k)$ provided $\bar{\mu}(\mathcal{E}_k)>0$. However, by definition of $\bar{\mu}$, we also have $\mu_k(\mathcal{E}_k)/K\le\bar{\mu}(\mathcal{E}_k)$. This implies that we must have $\bar{\mu}(\mathcal{E}_k)=0$.
\end{proof}
\begin{proposition}
\label{prop_transit}
    Let $\mu_1,\mu_2,\mu_3$ be distributions over $\mathcal{X}$ such that $\mu_1$ is $\sigma_1$-smooth w.r.t. $\mu_2$ and $\mu_2$ is $\sigma_2$-smooth w.r.t. $\mu_3$. Then $\mu_1$ is $\sigma_1\sigma_2$-smooth w.r.t. $\mu_3$.
\end{proposition}
\paragraph{Remark on notations:} Throughout the paper, we use lower case Greek letters $\mu,\nu$ to denote a probability measure over $\mathcal{X}$. For any two measures $\mu_1,\mu_2$, we use $\mu_1\cdot\mu_2$ to denote the product distribution of $\mu_1,\mu_2$ and $\mu^{\otimes T}$ to denote the $i.i.d.$ measure of $\mu$ over $\mathcal{X}^T$. We use boldface Greek letters $\pmb{\nu}^T$ to denote \emph{general} distributions over $\mathcal{X}^T$. We use Math Sans Serif font $\mathsf{P}$ to denote classes of distributions over $\mathcal{X}^T$. For any random process $X^T$ over $\mathcal{X}^T$, $t\le [T]$ and $\x^{t-1}$, we use $\nu_t(X_1\mid \x^{t-1})$ to denote the conditional distribution of $X_t$ conditioning on $\x^{t-1}$. We also use $\pmb{\nu}^T$ to denote the joint distribution of $X^T$ over $\mathcal{X}^T$. Sometimes, we write $\nu_t=\nu_t(X_t\mid \x^{t-1})$ to simplify the notation when the conditioning context $\x^{t-1}$ is clear. We should emphasize that all parameters appearing in our bounds are \emph{non-asymptotic}, meaning that one should \emph{not} view them as constants. We will often provide ranges of the parameters for our bounds to hold.

\vspace{-0.1in}
\section{Main results}

This is the main section of our paper. In Section~\ref{sec:alg_abs}, we study \emph{dynamic changing process} of cost $K$, i.e., the class $\mathsf{U}_K^1$, where we provide tight lower and upper bounds for finite VC-dimensional classes under absolute loss. We then refine these bounds for special losses, e.g., logarithmic loss in Section~\ref{sec:mixable}. In Section~\ref{sec:univ_smooth}, we analyze the class $\Tilde{\mathsf{U}}_1^{\sigma}$ (i.e., smooth processes with fixed but unknown reference measure) by establishing an important relation between $\Tilde{\mathsf{U}}_1^{\sigma}$ and the \emph{adversary K-selection} process introduced in Section~\ref{sec:univ_smooth}. We demonstrate the effectiveness of our approach by establishing sub-linear regrets for $1$-dimension threshold functions.

\vspace{-0.1in}
\subsection{The class $\mathsf{U}_K^1$ with finite VC class}
\label{sec:alg_abs}
Before we analyze the class $\mathsf{U}_K^1$, we note that the processes in $\mathsf{U}_K^1$ are highly \emph{non-stationary}. Our first main technical ingredient is the following decoupling of the random processes in $\mathsf{U}_K^1$ into $K$ (conditional) $i.i.d.$ processes.

\paragraph{Decoupling of $\mathsf{U}_K^1$:} Let $X^T$ be an arbitrary process in $\mathsf{U}_K^1$. We can \emph{extend} $X^T$ into another process $V^{KT}$ in the following manner. The first $T$ samples of $V^{KT}$ equal $X^T$. For any \emph{conditional} marginal $\nu_k$ of $X^T$ with $k\in [K]$, we extend the sample $X^{T}$ by sampling $i.i.d.$ from $\nu_k$ such that $\nu_k$ is used \emph{exactly} $T$ times in the sample $V^{KT}$ for each $k\in [K]$. Now, we denote $V^{(k)}=V_{k_1},\cdots,V_{k_T}$ as the subsequence in $V^{KT}$ that corresponds to $\nu_k$, where $k_t$s are \emph{random} indices.
\begin{proposition}
\label{prop:decouple}
    Conditioning on $k_1$ and $V^{k_1-1}$, the sample $V^{(k)}$ is an $i.i.d.$ process of length $T$ for all $k\in [K]$ (the $V^{(k)}$s are not necessarily independent for different $k$).
\end{proposition}
\begin{proof}
    Note that conditioning on $k_1$ and $V^{k_1-1}$, the distribution $\nu_k$ is determined. By definition of the conditional distribution for any events $A\subset \mathcal{X}^{T-1}$ and $B\subset \mathcal{X}$, we have
    \begin{align*}
        \mathrm{Pr}[V_{k_1}^{k_{T-1}}\in A,~V_{k_T}\in B\mid V^{k_1-1}]&=\mathrm{Pr}[V_{k_1}^{k_{T-1}}\in A\mid V^{k_1-1}]\cdot \mathrm{Pr}[V_{k_T}\in B\mid V_{k_1}^{k_{T-1}}\in A,~V^{k_1-1}]\\
        &=\mathrm{Pr}[V_{k_1}^{k_{T-1}}\in A\mid V^{k_1-1}]\cdot \nu_k(B),
    \end{align*}
    where $V_{k_1}^{k_{T-1}}=\{V_{k_1},V_{k_2},\cdots,V_{k_{T-1}}\}$. The proposition follows by induction on $T$.
\end{proof}

It is important to point out that the extension of $X^T$ to $V^{KT}$ is required for the decoupling to work. Otherwise, the constructed process $V^{(k)}$ is \emph{not} necessarily $i.i.d.$ (instead it is a \emph{random} prefix of an $i.i.d.$ process). Now, to analyze the performance of a predictor $\Phi$ on the process $X^T$, it is sufficient to study $\Phi$ on each of the sub-sequences $V^{(k)}$. Note that, this is generally a non-trivial task, since the predictor can only access to each of $V^{(k)}$s \emph{obliviously}, i.e., it never exactly knows the decoupling. The technical challenge is to ensure that the presence of other $V^{(k)}$s do not affect the performance of the predictor on each individual $V^{(k)}$.

\vspace{-0.1in}
\subsubsection{The adaptive epoch-EWA algorithm} 
The \emph{epoch} approach~\citep{lazaric2009hybrid} is a common way for dealing with \emph{distribution blind} (i.e., universal) cases. The algorithm proceeds as follows: we partition the time horizon into $\lceil\log T \rceil$ epochs, where each epoch $s$ ranges from time steps $2^{s}-1,\cdots,2^{s+1}$. In epoch $s$, we perform \emph{Exponential Weighted Average} (EWA) algorithm~\citep{lugosi-book} on a finite expert class by selecting one function from each equivalent class of $\mathcal{H}$ that agrees on the samples of the previous epochs. The rationale behind this approach is that as we obtain more and more samples, we can \emph{learn} the underlying hypothesis and then use the learned hypothesis to make prediction for the next epoch. However, this heavily relies on the assumption that the distributions are \emph{stationary} (i.e., the samples should have similar statistics among different epochs). This does not hold even for $\mathsf{U}_2^1$.
\begin{example}[Failure of epoch approach]
\label{exm:failureofepoch}
    Let $\mathcal{X}=\{\x_1,\x_2\}$ be the instance space and $\mathcal{H}=\{h_1,h_2\}$ be the hypothesis class with $h_1(\x_1)=h_2(\x_1)=1$, $h_1(\x_2)=0$ and $h_2(\x_2)=1$. We define distributions $\nu_1,\nu_2$ to be the singleton distributions on $\x_1$ and $\x_2$, respectively. We assume that the time horizon is $T=2^{s+1}-1$. For the first $s-1$ epochs, we use $\nu_1$ to generate samples and use $\nu_2$ for the last epoch. Now, after $s-1$ epochs, the algorithm, as in~\citep{lazaric2009hybrid}, will choose the expert to be any one of $h_1,h_2$ (since they agree on the previous samples). It is easy to see that the algorithm must incur at least $T/2$ regrets (the adversary simply labels the following samples using $h_i$ that differs from the algorithm's selection) .
\end{example}

\begin{algorithm}[h]
\caption{Adaptive epoch-EWA algorithm}\label{alg:1}
\textbf{Input}: Reference class $\mathcal{H}$ and update threshold $N$

 Let $s, E=0$ and $\mathcal{H}^0=\{h\}$, where $h\in \mathcal{H}$ is arbitrary
 
\For{$t=1,\cdots, T$}{
 Let $t_s=t$, $r=1$, $m= |\mathcal{H}^s|$ and $W^r=\{1,\cdots,1\}\in \mathbb{R}^m$
 
 \While{$E\le N$}{
   Set learning rate $\eta_r=\sqrt{8(\log m)/r}$
   
  Receive $\x_t$
  
  Make prediction
$$\hat{y}_t = \left(\sum_{i=1}^mh^s_i(\x_t)\cdot W_i^r\right)/\left(\sum_{i=1}^mW_i^r\right),~h^s_i\in \mathcal{H}^s$$
 Receive $y_t$
 
 Update $\forall i\le m,~W_i^{r+1}=W_i^{r}e^{-\eta_r \ell(h_i^s(\x_t),y_t)}$
 
 Set $$E=\max_{h\in \mathcal{H},~h^s\in\mathcal{H}^s}\left\{\sum_{e=0}^{r-1}1\{h(\x_{t_s+e})\not=h^s(\x_{t_s+e})\}:\forall j<t_s,~h(\x_j)=h^s(\x_j)\right\}$$
 Set $t=t+1, r=r+1$.
}
  Set $s=s+1$, $t=t-1$ and $E=0$
  
  Define equivalence $h_1\sim_s h_2$ if $\forall j\le t,~h_1(\x_j)=h_2(\x_j)$, where $h_1,h_2\in \mathcal{H}$.
  
  Let $\mathcal{H}^s$ be the class that selects exactly one element from each equivalent class under $\sim_s$.
}
\end{algorithm}
It can be shown that \emph{any predefined set of} epochs cannot provide bounds better than $\Omega(T^{2/3})$, even for the simple class of Example~\ref{exm:failureofepoch} (see Example~\ref{exm:failureofepoch1} in Appendix~\ref{sec:proofthm8}). Our main idea for resolving this issue is the \emph{adaptive} epoch approach, presented in Algorithm~\ref{alg:1}. Note that the "adaptive" in Algorithm~\ref{alg:1} has two different meanings. First, we select the learning rate $\eta_r$ adaptively, and second, the error bound $E$ is computed adaptively (i.e., we change the epochs according to the samples we observe). Our main result for this section is the following performance bound of Algorithm~\ref{alg:1}.

\begin{theorem}
\label{thm_sqrt_bound}
    Assume that the loss $\ell$ is convex in the first argument and upper bounded by $1$, and $\mathcal{H}\subset \{0,1\}^{\mathcal{X}}$ is a class of finite VC-dimension. If $\hat{y}_t$ is the prediction rule of Algorithm~\ref{alg:1} that takes input $\mathcal{H}$ and $N=\sqrt{(T\cdot \vch\log T)/K}$, we have for all $\epsilon>0$ if $K^3\cdot \vch \le O(T^{1-\epsilon}/\log T)$ 
    $$\tilde{r}_T(\mathcal{H},\mathsf{U}_K^1)\le O(\sqrt{KT\cdot\vch\log T}),$$
    where $O$ hides a constant that depends only linearly on $1/\epsilon$.
    Furthermore, for any numbers $d,K\ge 1$ with $(8Kd)\cdot\log(2Kd)\le T$, we have
    $$\sup_{\mathcal{H},\vch\le d}\tilde{r}_T(\mathcal{H},\mathsf{U}_K^1)\ge \sqrt{KdT/64},$$
    under the absolute loss. For any $K\le T$ the bound $\Omega(\sqrt{KT})$ holds for threshold functions.
\end{theorem}
\begin{proof}[Sketch of Proof]
    We only sketch the main idea here and refer to Appendix~\ref{sec:proofthm8} for detailed proof. At a high level, our goal is to bound the number of epochs (i.e., the number of times we reenter the while loop). Note that, we are exiting the while loop only when the approximation error $E$ of current expert class $\mathcal{H}^s$ is larger than the threshold $N$. Suppose we can upper bound the number of epochs by $S$. We denote $T_1,\cdots,T_S$ to be the length of each epoch. Note that for each epoch $s\le S$, the regret can be split into two parts: the regret against expert class $\mathcal{H}^s$ and the error of approximating  $\mathcal{H}$ by $\mathcal{H}^s$. For the first term, we have by standard result~\citep[Thoerem 2.3]{lugosi-book} that the regret is upper bounded by $\sqrt{2T_s|\mathcal{H}^s|}\le \sqrt{2T_s\cdot\vch \log T}$, the last inequality follows from $|\mathcal{H}^s|\le T^{\vch}$. The second term is trivially upper bounded by $N$, since we change epochs once the approximation error is larger than $N$. Therefore  the regret is upper bounded by $\sum_{s=1}^S(\sqrt{2T_s\cdot\vch\log T}+N)\le SN+\sqrt{2ST\cdot\vch \log T}$, where the inequality follows from Cauchy–Schwarz inequality $\sum_{s=1}^S\sqrt{T_s}\le \sqrt{S\sum_{s=1}^nT_s}=\sqrt{ST}$. The key technical challenge is to show that if we choose $N=\sqrt{(T\cdot\vch\log T)/K}$, we can ensure that $S\le O(K)$ w.h.p. under \emph{any} process in $\mathsf{U}_K^1$, provided $K^3\cdot \vch \le O(T^{1-\epsilon}/\log T)$. This is achieved using the decoupling of $\mathsf{U}_K^1$, together with a symmetric argument for bounding the approximation errors on each of the decoupled sub-sequences, see Lemma~\ref{lem_exchange_bound} and~\ref{lem_bad_event} in Appendix~\ref{sec:proofthm8}. 
    
    To prove the lower bound, we use a \emph{hard} hypothesis class similar to~\citep{haghtalab2022smoothed}, together with a \emph{mixed} adversary-$i.i.d.$ process based on the concept of \emph{Littlestone forests} that achieves the tightest dependency $\Omega(\sqrt{K\cdot\vch T})$. We note that a reduction to the Littlestone dimension as in~\citep{haghtalab2022smoothed} can only provide an $\Omega(\sqrt{KT})$ bound. Our technical contribution is to obtain a tight dependency on both $\vch$ and $K$. See Appendix~\ref{sec:proofthm8} for detailed proof.
\end{proof}
\vspace{-0.18in}
\begin{remark}
    Note that, for $K=1$, Theorem~\ref{thm_sqrt_bound} recovers the upper bound in~\citep{lazaric2009hybrid} with lower computational cost (we only run $O(1)$ epochs for $K=1$, while~\cite{lazaric2009hybrid} runs $O(\log T)$ epochs). We believe the condition $K^3\cdot \vch \ll T^{1-\epsilon}/\log T$ is an artifact of our analysis and could be eliminated via a further refined approach. We will establish a tighter dependency on $K$ for the full range $K\le T$ in the next section with a slightly worse $\log^3 T$ factor. Furthermore, Algorithm~\ref{alg:1} can be made adaptive to $K$ as well, see Remark~\ref{remark:adptivek} (in Appendix~\ref{sec:proofthm8}). 
    Theorem~\ref{thm_sqrt_bound} also establishes a fundamental distinction between the universal and distribution aware case, as in Corollary~\ref{cor:k_smooth} (in Appendix~\ref{sec:known_dist}) w.r.t  dependency of $K$, i.e., $K$ vs $\log K$.
\end{remark}

\vspace{-0.1in}
\subsubsection{Improved bounds through stochastic sequential cover}
\label{sec:mixable}
The adaptive epoch approach proposed in the previous section results in tight bounds for the absolute loss and general convex bounded losses. For some special losses such as  the logarithmic loss and general mixable losses, we provide tighter bounds on regret. We note that our results in this section also provide tighter bounds for bounded convex losses with parameters beyond the ranges of Theorem~\ref{thm_sqrt_bound}. We start with the following generic upper bounding technique:

\paragraph{A generic upper bounding technique:} 
A crucial part of establishing regret bounds when the reference distribution is known  (e.g.,~\cite{haghtalab2022smoothed}), as  discussed in Appendix~\ref{sec:known_dist}, is to apply the EWA algorithm over a \emph{uniform} cover of $\mathcal{H}$ (see Corollary~\ref{cor1}). This, unfortunately, is not available for our \emph{universal} case, since we do not know the reference measure $\mu$ in advance. A general methodology for dealing with such cases was introduced recently in~\citep{wu2022expected} via the so called  \emph{stochastic sequential cover}.
\begin{definition}
\label{def-gcover}
We say a class $\mathcal{G}$ of functions $\mathcal{X}^*\rightarrow [0,1]$ (where $\mathcal{X}^*$ is the set of all finite sequences over $\mathcal{X}$)
is a stochastic global sequential cover of a class $\mathcal{H}\subset [0,1]^{\mathcal{X}}$ 
w.r.t. the class $\mathsf{P}$ of distributions over $\mathcal{X}^T$ 
at scale $\alpha>0$ and confidence $\beta>0$, if for all $\pmb{\nu}^T\in \mathsf{P}$, 
$$\mathrm{Pr}_{\x^T\sim \pmb{\nu}^T}\left[\exists h\in \mathcal{H}~\forall 
g\in \mathcal{G}~\exists~t\in [T]~s.t.~|h(\x_t)-g(\x^t)|>\alpha\right]\le \beta.$$
\end{definition}
This definition immediately implies the following regret bounds by the standard expert algorithms (e.g., EWA), as in \citep[Theorem 3 \& 4]{wu2022expected}. Appendix~\ref{sec:proofpro9lem10} presents the proof.
\begin{proposition}
    \label{pro:cover2regret}
    Let $\mathcal{G}$ be a stochastic sequential cover of $\mathcal{H}$ w.r.t $\mathsf{P}$ at scale $\alpha=0$ and confidence $\beta=\frac{1}{T}$. Then $\Tilde{r}_T(\mathcal{H}, \mathsf{P})\le O(\sqrt{T\log|\mathcal{G}|})$ under bounded convex losses and $\Tilde{r}_T(\mathcal{H},\mathsf{P})\le \log|\mathcal{G}|$ under logarithmic loss and bounded mixable losses.
\end{proposition}
The above results lead us to the following general approach for upper bounding $\Tilde{r}_T$ through stochastic sequential cover. Let $\mathcal{H}\subset \{0,1\}^{\mathcal{X}}$ and $\mathsf{P}$ be \emph{arbitrary} classes as defined above. We first find a prediction rule $\Phi:(\mathcal{X}\times \{0,1\})^*\times \mathcal{X}\rightarrow \{0,1\}$ such that:
\begin{equation}
\label{eq:realizbalebound}
    \forall \pmb{\nu}^T\in \mathsf{P},~\mathrm{Pr}_{\x^T\sim\pmb{\nu}^T}\left[\sup_{h\in \mathcal{H}}\mathsf{err}(\Phi,h,\x^T)\ge B(T, \beta)\right]\le \beta,
\end{equation}
where $\mathsf{err}(\Phi,h,\x^T)=\sum_{t=1}^T1\{\Phi(\x^t,h(\x_1),\cdots,h(\x_{t-1}))\not=h(\x_t)\}$ is the cumulative error of $\Phi$ under the \emph{realizable} sample of $h$ on $\x^T$ and $B(T,\beta)$ is an error bound depending on the confidence parameter $\beta$ and the time horizon $T$. For any such prediction rule $\Phi$, we can then bound the stochastic sequential cover using the following lemma as in~\cite[Lemma 8]{wu2022expected}, see also~\cite[Lemma 12]{ben2009agnostic}. Appendix~\ref{sec:proofpro9lem10} presents the proof.
\begin{lemma}
\label{lem:error2cover}
    Let $\mathcal{H}$ and $\mathsf{P}$ be arbitrary classes and $\Phi$ be a predictor satisfying (\ref{eq:realizbalebound}). Then there exists a stochastic sequential cover $\mathcal{G}$ of $\mathcal{H}$ w.r.t. $\mathsf{P}$ at scale $\alpha=0$ and confidence $\beta$ such that $\log |\mathcal{G}|\le O((B(T,\beta)+1)\cdot \log T)$.
\end{lemma}
The upper bound on $\Tilde{r}_T(\mathcal{H},\mathsf{P})$  then follows from Proposition~\ref{pro:cover2regret}. We remark that a crucial part for applying this approach is finding the predictor $\Phi$ and the upper bound $B(T,\beta)$, which is generally non-trivial if the processes in $\mathsf{P}$ are \emph{non-stationary} due to the  requirement of finding a bound on the form $\mathrm{Pr}[\sup_h]$.

\paragraph{The product distributions:}
\label{app-symmetry}
We first consider a simpler distribution class and illustrate how our technique works. We say a distribution $\pmb{\nu}^T$ over $\mathcal{X}^T$
is a product distribution of type $K$ if there exist distributions 
$\nu_1,\cdots,\nu_K$ over $\mathcal{X}$ such that $\pmb{\nu}^T=\prod_{t=1}^T\nu_t$, where $\nu_t\in \{\nu_1,\cdots,\nu_K\}$. Note that distributions $\nu_k$s and the configuration of the marginals of $\pmb{\nu}^T$ need not be fixed and are \emph{unknown} to the learner (e.g., the processes in Example~\ref{exm:failureofepoch} are product distributions of type $2$). We prove the following upper bound for the stochastic sequential covering for such distributions:

\begin{theorem}
\label{th10}
Let $\mathcal{H}$ be a binary valued class with finite VC-dimension, 
and $\mathsf{P}$ be the class of all product distributions over $\mathcal{X}^T$ with type $K$. 
Then, there exists a global sequential covering set $\mathcal{G}$ of $\mathcal{H}$ 
at scale $\alpha=0$ and confidence $\beta$ such that $\log |\mathcal{G}|\le O(K\cdot\mathsf{VC} (\mathcal{H})\log^2 T+\log T\log(1/\beta)).$
\end{theorem}
\begin{proof}[Sketch of Proof]
    The main idea is to apply the generic upper bounding technique. To do so, we show that for the \emph{1-inclusion graph} predictor $\Phi$~\citep{haussler1994predicting}, one can upper bound the \emph{realizable} cumulative error $B(T,\beta)\le O(K\cdot\vch\log(T/K)+\log(1/\beta))$, as in (\ref{eq:realizbalebound}). The main technical difficulty is in establishing a high probability error bound of form $\mathrm{Pr}[\sup_h]$ for $\Phi$. This is established by exploiting the \emph{permutation invariance} of $\Phi$ similar to~\cite[Lemma 7]{wu2022expected}, but with more carefully designed permutations. The bound for the sequential covering then follows by Lemma~\ref{lem:error2cover}. See Appendix~\ref{sec:proofthm10} for detailed proof.
\end{proof}
Thus upper bounds on the regret follow from Theorem~\ref{th10} and Proposition~\ref{pro:cover2regret}.
\begin{corollary}
    Let $\mathcal{H}$ be a binary valued class of finite VC-dimension and $\mathsf{P}$ be the class of all production distributions of type $K$. For  any $K,T\ge 1$ we have  $\tilde{r}_T(\mathcal{H},\mathsf{P})\le O(\sqrt{KT\cdot\vch\log^2 T})$ under bounded convex losses and $\tilde{r}_T(\mathcal{H},\mathsf{P})\le O(K\cdot\vch\log^2 T)$ under log-loss.
\end{corollary}

\paragraph{The class $\mathsf{U}_K^1$:} The $1$-inclusion graph algorithm for \emph{product} processes in the previous part relies heavily on  symmetries in the product distribution. This, unfortunately, does not hold for general processes in $\mathsf{U}_K^1$ (e.g., the hard instance constructed in the lower bound proof of Theorem~\ref{thm_sqrt_bound}). Our main technique to deal with this issue is to replace the 1-inclusion graph predictor with the ERM rule, together with a \emph{perturbation} argument for establishing a realizable cumulative error bound, as in (\ref{eq:realizbalebound}). This allows us to establish the following stochastic sequential covering bound:

\begin{theorem}
\label{universal_k}
   Let $\mathcal{H}$ be a binary valued class of finite VC-dimension. Then there exists a stochastic sequential covering set $\mathcal{G}$ of $\mathcal{H}$ w.r.t. $\mathsf{U}^{1}_K$  at scale $\alpha=0$ and confidence $\beta>0$ such that
    $$\log|\mathcal{G}|\le O(K(\vch\log^3 T+\log^2 T\log(K/\beta))\log(\vch\log T\log(K/\beta))),$$
    where $O$ hides absolute constant independent of $K,\vch,T,\beta$.
\end{theorem}

\begin{proof}[Sketch of Proof]
    We sketch only the high level idea here and refer to Appendix~\ref{sec:proofthm13} for the full proof. We show that for any process in $\mathsf{U}_K^1$ and the \emph{ERM predictor} $\Phi$, the realizable cumulative error (see Equation~(\ref{eq:realizbalebound})) is upper bounded by $B(T,\beta)\le O(K(\vch\log^2 T+\log T\log(K/\beta))\cdot \Delta)$, where $\Delta=\log(\vch\log T\log(K/\beta))$. To achieve this, we first decouple the process in $\mathsf{U}_K^1$ into $K$ \emph{conditional} $i.i.d.$ processes (Proposition~\ref{prop:decouple}). We then establish the realizable cumulative error bound on each of the decoupled sub-sequences (which are \emph{conditional} $i.i.d.$). The key technical justification that allows us to do so is that an ERM rule with additional \emph{realizable} samples is still an ERM rule. This allows us to bound the cumulative error for each decoupled sub-sequence \emph{independently} even though we can only access them \emph{obliviously}. We emphasize that to bound the realizable cumulative error for ERM rule even for $i.i.d.$ process is still a non-trivial task, since we require a $\mathrm{Pr}[\sup_h]$ type bound for Lemma~\ref{lem:error2cover} to apply. To resolve this issue we introduce a novel \emph{perturbation} argument, as presented in Lemma~\ref{lem_infinite2finite} (Appendix~\ref{sec:proofthm13}), which provides a generic way of converting a $\sup_h\mathrm{Pr}$ bound to a $\mathrm{Pr}[\sup_h]$ bound for any finite VC class with $i.i.d.$ sampling.
\end{proof}

We now have the following regret bounds for VC-class, see Appendix~\ref{sec:proofthm13} for detialed proof.

\begin{corollary}
\label{cor:mixable}
    For VC class $\mathcal{H}$ we have $\tilde{r}_T(\mathcal{H},\mathsf{U}_K^1)\le O\left(\sqrt{\Delta\cdot KT\cdot\vch\log^3 T}\right)$ under bounded convex losses and $\tilde{r}_T(\mathcal{H},\mathsf{U}_K^1)\le O(\Delta\cdot K\cdot\vch\log^3 T)$ under log-loss and bounded mixable losses, where $\Delta=\log(\vch\log(KT))$. Moreover, for $Kd\ll T/\log d$, we have $\sup_{\mathcal{H},\vch\ge d}\Tilde{r}_T(\mathcal{H},\mathsf{U}_K^1)\ge d\max\{K,\log(T/d)\}$ under log-loss.
\end{corollary}

\vspace{-0.1in}
\subsection{The class $\Tilde{\mathsf{U}}_1^{\sigma}$ with threshold functions}
\label{sec:univ_smooth}
We now study the universal smooth process $\Tilde{\mathsf{U}}_1^{\sigma}$ with fixed (but unknown) reference measure, where $\sigma\in (0,1]$ is any positive real~\footnote{Note that, the classes $\mathsf{U}_K^1$ and $\Tilde{\mathsf{U}}_1^{\sigma}$ \emph{do not} include each other, for all $\sigma\in (0,1)$.}. We start with the following reduction. Let $\mu$ be an arbitrary distribution over $\mathcal{X}$. We say a random variable $X$ is $K$-selection w.r.t. $\mu$ if there exists a \emph{deterministic} function $f$ such that $X=f(V^K)\in \{V_1,\cdots,V_K\}$, where $V^K\sim \mu^{\otimes K}$. We say a \emph{random process} over $\tilde{X}^T$ is \emph{adversary $K$-selection} w.r.t. $\mu$ if for all $t\le T$ the \emph{conditional} marginals $\nu_t(X_t\mid X^{t-1})$ are $K$-selection w.r.t. $\mu$ almost surely. In Appendix~\ref{sec:proofsmooth}, we prove the following key lemma that relates the class $\Tilde{\mathsf{U}}_1^{\sigma}$ to  the adversary $K$-selection processes.

\begin{lemma}
\label{lem_smooth_2_select}
    Let $A\subset \mathcal{X}^T$ be any event. If for all adversary $K$-selection process $\tilde{X}^T$ we have $\mathrm{Pr}[\Tilde{X}^T\in A]\ge 1-\beta$,
    then  for any $\sigma$-smooth process $X^T\in \tilde{\mathsf{U}}_1^{\sigma}$ we have $\mathrm{Pr}[X^T\in A]\ge 1-2\beta$, provided $K\ge \frac{\log(T/\beta)}{\sigma}$.
\end{lemma}

    Lemma~\ref{lem_smooth_2_select} shows that to bound the prediction performance for $\Tilde{\mathsf{U}}_1^{\sigma}$ it is sufficient to bound the performance of the adversary $K$-selection processes. Perhaps surprisingly, this reduction essentially loses no information, since the adversary $K$-selection processes are also $\in\Tilde{\mathsf{U}}_1^{1/K}$. This follows from the fact that for any event $A$ we have $\mathrm{Pr}[f(V^K)\in A]\le 1-(1-\mathrm{Pr}_{V\sim \mu}[V\in A])^K\le K\mathrm{Pr}_{V\sim \mu}[V\in A]$, i.e., the conditional marginals $\nu_t$ must be $1/K$-smooth w.r.t. $\mu$.

Our main result of this section is the following stochastic sequential covering bound for the threshold functions w.r.t. adversary $K$-selection processes. See Appendix~\ref{sec:proofsmooth} for a detailed proof.

\begin{theorem}
\label{thm_threhold_smooth}
    Let $\mathcal{H}=\{h_a(x)=1\{x\ge a\}:x,a\in [0,1]\}$ be the class of $1$-dimension threshold functions and $\mathsf{P}$ be the class of all adversary $K$-selection processes. Then there exists a stochastic sequential covering set $\mathcal{G}$ w.r.t. $\mathsf{P}$ at scale $\alpha=0$ and confidence $\beta>0$ such that
    $$\log |\mathcal{G}|\le O(\sqrt{KT\log(2KT^2/\beta)}).$$
\end{theorem}
\begin{proof}[Sketch of Proof]
    We sketch the main idea here and refer to Appendix~\ref{sec:proofsmooth} for a detailed proof. We stress that even though the threshold functions may be simple from classical learning theory perspective, the proof of Theorem~\ref{thm_threhold_smooth} is not. This is due to the complex structure of adversary $K$-selection processes. Our proof follows a similar path as in~\citep[Thm 13]{wu2022expected} but with a substantially more sophisticated analysis. To do so, we maintain a \emph{realization} tree, with each node of the tree labeled by a subset of $\mathcal{H}$. We expand the leaves of the tree every time we receive a sample $\Tilde{X}_t$ by splitting the associated subset of $\mathcal{H}$ according to the labels on $\Tilde{X}_t$. Our main technical contribution is to bound the \emph{maximum} depth of the realization tree to be $O(\sqrt{KT\log(2KT^2/\beta})$ w.p. $\ge 1-\beta$. This relies on a careful analysis on the splitting process. The bound for the stochastic sequential covering will then follow from a similar construction as in~\citep[Thm 13]{wu2022expected}.
\end{proof}
We complete this section with the following bounds for the regret.
\begin{corollary}
\label{cor_threhold_smooth}
    Let $\mathcal{H}=\{h_a(x)=1\{x\ge a\}:x,a\in [0,1]\}$, then 
    $$\tilde{r}_T(\mathcal{H},\tilde{\mathsf{U}}_1^{\sigma})\le O\left(\sqrt{({T}/{\sigma})\log^2(T/\sigma)}\right),$$
    under bounded mixable losses and logarithmic loss. For bounded convex losses, we have
    $$\tilde{r}_T(\mathcal{H},\tilde{\mathsf{U}}_1^{\sigma})\le O\left(\sqrt{\frac{T^{3/2}\log (T/\sigma)}{\sigma^{1/2}}}\right).$$
\end{corollary}
\begin{proof}
    This follows directly by Theorem~\ref{thm_threhold_smooth}, Lemma~\ref{lem_smooth_2_select} and Proposition~\ref{pro:cover2regret}.
\end{proof}

\begin{remark}
     Corollary~\ref{cor_threhold_smooth} establishes sublinear regrets as  long as $\sigma^{-1} \ll T/\log^2 T$. Our lower bounds in Theorem~\ref{thm_sqrt_bound} imply $\Omega(\sqrt{T/\sigma})$ lower bound for absolute loss and $\Omega(\frac{1}{\sigma})$ for log-loss. This indicates that our upper bounds here may not be tight. We leave it as an open problem to obtain sublinear (and tight) regret for general finite VC-classes under $\Tilde{\mathsf{U}}_1^{\sigma}$. We stress that this is a hard task, since in the proof of Theorem~\ref{thm_threhold_smooth} we have exploited non-trivial properties of threshold functions that seem  to be not easily generalizable to general VC-class.
\end{remark}

\bibliography{bibliography.bib}

\appendix

\section{Preliminaries: Distribution aware case}
\label{sec:known_dist}
We discuss the classical smooth adversary case, as introduced in~\citep{haghtalab2020smoothed,haghtalab2022smoothed,block2022smoothed}, when the reference measure is \emph{known} in advance. We present an alternate view here, which is easier to adapt to more general losses, e.g., logarithmic loss. 

By Proposition~\ref{prop_equiv} and~\ref{prop_transit}, we know that analysis of the smooth adversary case with multiple (known) reference measures can be reduced to the case with only \emph{one} reference measure. It is therefore sufficient to consider the setup from Example~\ref{exm_smooth} with a single $\mu$.

We start with the following key proposition due to~\citep{haghtalab2022smoothed} (and simplified substantially in~\citep{block2022smoothed}). We note that this proposition will also be used in the \emph{universal} reference measure case discussed in Section~\ref{sec:univ_smooth}.

\begin{proposition}
\label{prop1}
    For any $\sigma$-smooth random process $X^T$ with reference measure $\mu$, there exists a (coupled) random processes $V^{mT}$ with $i.i.d.$ distribution $\mu^{\otimes mT}$ such that w.p. $\ge 1-Te^{-\sigma m}$ (over the joint distribution of $X^T,V^{mT}$), we have
    $$\forall t\in [T],~X_t\in \{V_{m(t-1)+1},\cdots, V_{mt}\}$$
\end{proposition}
\begin{proof}
    We first sample $V^{mT}$ according to the $i.i.d.$ distribution $\mu^{\otimes mT}$. We then construct $X_t$ recursively in the following manner. After generating  $X_1,\cdots,X_{t-1}$, the conditional distribution of $\nu(X_t\mid X^{t-1})$ is determined. Let $S_t$ be a random set such that each $Z_i\in\{V_{m(t-1)+1},\cdots,V_{mt}\}$ is included into $S_t$ independently w.p. $\sigma v_t(Z_i)$ (i.e., w.p. $1-\sigma v_t(Z_i)$ we do not include it), where $v_t$ is the density of $\nu(X_t\mid X^{t-1})$ w.r.t. $\mu$ (see Example~\ref{exm_smooth}). We then generate $X_t$ by sampling uniformly from $S_t$ if $S_t$ is non-empty and sampling independently from $\nu_t$ if $S_t$ is empty. It is easy to verify that $X^T$ is distributed according to $\pmb{\nu}^T$, and w.p. $\ge 1-(1-\sigma)^m$, we have $X_t\in \{V_{m(t-1)+1},\cdots,V_{mt}\}$. The result follows by union bound on $[T]$.
\end{proof}

A set $A\subset\mathcal{X}^{\infty}$ is monotone if for any $\x^T\subset \z^{T'}$, we have $\x^T\in A\Rightarrow \z^{T'}\in A$, where $\x^T\subset \z^{T'}$ mean $\x^T$ is a \emph{sub-sequence} of $\z^{T'}$ and $\x^T\in A$ means any infinite sequence with \emph{prefix} $\x^T$ is in $A$. We have the following lemma. Note that~\cite{haghtalab2022smoothed} used a similar idea as the following lemma but in a different form.
\begin{lemma}
\label{lem_monotone}
    Let $X^T$ and $V^{mT}$ be the coupling as in Proposition~\ref{prop1} and $A\subset\mathcal{X}^{\infty}$ be an arbitrary monotone set, then
    $$\mathrm{Pr}[X^T\in A]\le Te^{-\sigma m}+
    \mathrm{Pr}[V^{mT}\in A].$$
\end{lemma}
\begin{proof}
    By Proposition~\ref{prop1}, we have w.p. $\ge 1-Te^{-\sigma m}$ that $X^T\subset V^{mT}$. Denote $B$ to be such an event. Since $A$ is monotone, we have
    $$\mathbb{E}[1\{\{X^T\in A\}\wedge B\}-1\{\{V^{mT}\in A\}\wedge B\}]\le 0.$$ This implies
    $$\mathrm{Pr}[\{X^T\in A\}\wedge B]\le \mathrm{Pr}[\{V^{mT}\in A\}\wedge B]\le \mathrm{Pr}[V^{mT}\in A].$$
    Our result follows by observing that:
    \begin{align*}
        \mathrm{Pr}[X^T\in A]&=\mathrm{Pr}[\{X^T\in A\}\wedge B]+\mathrm{Pr}[\{X^T\in A\}\wedge \bar{B}]\\
        &\le \mathrm{Pr}[\{X^T\in A\}\wedge B]+\mathrm{Pr}[\bar{B}]\le \mathrm{Pr}[\{X^T\in A\}\wedge B]+Te^{-\sigma m}
    \end{align*}
\end{proof}

Note that unions and intersections of \emph{any} collection of monotone sets are monotone. For any two functions $h_1,h_2:\mathcal{X}\rightarrow \{0,1\}$, the set $A_N=\{\x^{\infty}\in \mathcal{X}^{\infty}:\sum_{t=1}^{\infty}1\{h_1(\x_t)\not=h_2(\x_t)\}\ge N\}$ is monotone for all $N\in \mathbb{N}$.

We now present one of our key technical lemma that improves a $\log T$ term when compared to~\cite[Lemma B.2]{haghtalab2022smoothed}, which is crucial to establish tight bounds for mixable losses, e.g., logarithmic loss. This will also be key for our $\sup_h\mathbb{E}$ to $\mathbb{E}\sup_h$ conversion technique, as established in Appendix~\ref{sec:proofthm13}.

\begin{lemma}
\label{lem_k_tail}
   Let $\mathcal{H}\subset \{0,1\}^{\mathcal{X}}$ be any class with finite VC-dimension and $\mu$ be an arbitrary probability measure over $\mathcal{X}$. If $\mathcal{F}_{\epsilon}$ is an $\epsilon$-cover of $\mathcal{H}$ w.r.t. $\mu$, i.e.,
   \begin{equation}
   \label{eq-cw1}
\sup_{h\in \mathcal{H}}\inf_{f\in \mathcal{F}_{\epsilon}}\mathrm{Pr}_{\x\sim\mu}[h(\x)\not=f(\x)]\le \epsilon,
\end{equation}
with $\epsilon=\frac{1}{2M^2}$, then for all $n\in \mathbb{N}$ and $M\ge 2$ we have: $$\mathrm{Pr}_{\x^M\sim \mu^{\otimes M}}\left[\sup_{h\in \mathcal{H}}\inf_{f\in \mathcal{F}_{\epsilon}}\sum_{t=1}^{M}1\{h(\x_t)\not=f(\x_t)\}\ge 3\mathsf{VC}(\mathcal{H})+n\right]\le \frac{2}{M^n}.$$
\end{lemma}
\begin{proof}
    For any $h\in \mathcal{H}$, we denote by $\hat{f}_h=\arg\min_{f\in \mathcal{F}_{\epsilon}}\mathrm{Pr}_{\x\sim \mu}[h(\x)\not=f(\x)]$. Let $S^0$ and $S^1$ be $i.i.d.$ samples of $\mu$ with size $M$ and $M^2$, respectively. For any $N\le M$, we define two events:
    $$A_1^N=\left\{\exists h\in \mathcal{H}~s.t.~\sum_{s\in S^0}1\{h(s)
\not=\hat{f}_h(s)\}\ge N\right\},
$$ 
and
$$
A_2^N=\left\{\exists h\in \mathcal{H}~s.t.~\sum_{s\in S^0}1\{h(s)\not=\hat{f}_h(s)\}\ge 
N~\text{ and }~\sum_{s\in S^1}1\{h(s)\not=\hat{f}_h(s)\}= 0\right\}.
$$
We now claim that $\mathrm{Pr}[A_2^N\mid A_1^N]\ge \frac{1}{2}$. To see this, conditioning on $A_1^N$, there exists some $h$ for $A_1^N$ to happen. For such function $h$, we can select $\epsilon=1/(2M^2)$ in (\ref{eq-cw1}) such that (since $|S^1|= M^2$):
$$\mathbb{E}\left[\sum_{s\in S^1}1\{h(s)\not=\hat{f}_h(s)\}\right]\le \frac{1}{2}.$$
By the First Moment  method we know that $1-\mathrm{Pr}[X=0]=\mathrm{Pr}{[X\ge 1]} \le \mathbb{E}[X]\le 1/2$
 for any random variable $X$ supported on $\mathbb{N}$ with $\mathbb{E}[X]\le 1/2$.
 Thus $\mathrm{Pr}[A_2^N\mid A_1^N]\ge \frac{1}{2}$. This implies that $\mathrm{Pr}[A_1^N]\le 2\mathrm{Pr}[A_1^N\cap A_2^N]\le 2\mathrm{Pr}[A_2^N]$.

We now upper bound $\mathrm{Pr}[A_2^N]$. By symmetries of $i.i.d.$ distribution, we have $\mathrm{Pr}[A_2^N(S^0\cup S^1)]=\mathbb{E}_{\pi}\mathrm{Pr}[A_2^N(\pi(S^0\cup S^1))]\le \sup_{S^0\cup S^1}\mathrm{Pr}_{\pi}[A_2^N(\pi(S^0\cup S^1))]$, where $\pi$ is uniform random permutation over $S^0\cup S^1$. We now fix any $S^0\cup S^1$ and perform a uniform random permutation $\pi$. Let $h\in \mathcal{H}$ be any function such that there exist at least $N$ elements in $S^0\cup S^1$ for which $\hat{f}_h(s)\not=h(s)$ (otherwise $\mathrm{Pr}_{\pi}[A_2^N]=0$). Note that, in order for $A_2^N$ to happen under $\pi$, none of the elements $s\in S^0$ for which $\hat{f}_h(s)\not=h(s)$ should be permuted to $S^1$. Denote such an event to be $B$. We have
$$\mathrm{Pr}_{\pi}[B]=\frac{\binom{M}{N}}{\binom{M^2+M}{N}}\le \frac{1}{M^{N}},$$
where we have used the fact that $\frac{a}{b}\ge \frac{a-i}{b-i}$ for all $b\ge a\ge i> 0$. Since there are at most $(M^2+M)^{\mathsf{VC}(\mathcal{H})}$ functions restricted on $S^0\cup S^1$, we have by union bound that
$$\mathrm{Pr}_{\pi}[A_2^N]\le \frac{(M^2+M)^{\mathsf{VC}(\mathcal{H})}}{M^N}\le M^{3\vch-N},$$
where we used the fact that $M\ge 2$. The result follows by taking $N:=3\vch+n$ in the above expression and noting that $\mathrm{Pr}[A_1^N]\le 2\mathrm{Pr}[A_2^N]$.
\end{proof}

Lemma~\ref{lem_k_tail} implies the following important approximating bound for $\sigma$-smooth processes.

\begin{proposition}
\label{them_k_tail}
    Let $\mathcal{H}\subset \{0,1\}^{\mathcal{X}}$ be a class with finite VC-dimension, $\mu$ be an arbitrary distribution over $\mathcal{X}$ and $X^T$ be any $\sigma$-smooth random process w.r.t. $\mu$. If we take $\epsilon=\frac{\sigma^2}{2T^2\log^2(T/\beta)}$ for some $\beta>0$ and  $\mathcal{F}_{\epsilon}$ to be the $\epsilon$-covering set of $\mathcal{H}$ w.r.t. $\mu$ as in Lemma~\ref{lem_k_tail}, then
    $$\mathrm{Pr}\left[\sup_{h\in \mathcal{H}}\inf_{f\in \mathcal{F}_{\epsilon}}\sum_{t=1}^T1\{h(X_t)\not=f(X_t)\}\ge 3\vch+n\right]\le \beta+\frac{2}{T^n}.$$
\end{proposition}
\begin{proof}
    Taking $m=\frac{\log(T/\beta)}{\sigma}$ as in Proposition~\ref{prop1} one can make the error probability upper bounded by $\beta$. Let $M=mT$ as in Lemma~\ref{lem_k_tail}, we have by setting $\epsilon=\frac{1}{2M^2}=\frac{\sigma^2}{2T^2\log^2(T/\beta)}$ the probability as in Lemma~\ref{lem_k_tail} is upper bounded by $\frac{2}{T^n}$ since $M\ge T$. The theorem follows by Lemma~\ref{lem_monotone} by noticing that the event of the proposition is monotone (see the discussion follows Lemma~\ref{lem_monotone} by noticing that $\sup\inf\equiv \cup\cap$) and we apply Lemma~\ref{lem_k_tail} over the process $V^{mT}$.
\end{proof}

\begin{corollary}
\label{cor1}
    Let $\mathcal{H}\subset \{0,1\}^{\mathcal{X}}$ be a binary valued class with finite VC-dimension, and $\mu$ be arbitrary distributions over $\mathcal{X}$. For any convex and bounded loss, we have $$\tilde{r}_T(\mathcal{H},\mathsf{S}^{\sigma}(\mu))\le O\left(\sqrt{T\cdot \mathsf{VC}(\mathcal{H})\log (T/\sigma)}+\mathsf{VC}(\mathcal{H})\right).$$ For Log-loss and bounded mixable loss we have $$\tilde{r}_T(\mathcal{H},\mathsf{S}^{\sigma}(\mu))\le O(\mathsf{VC}(\mathcal{H})\log(T/\sigma)).$$
\end{corollary}
\begin{proof}
    Let $\epsilon$ be as in Proposition~\ref{them_k_tail} and $\beta=\frac{1}{2T}$. Taking $n=2$, we have the tail probability in Proposition~\ref{them_k_tail} upper bounded by $\frac{1}{T}$. Applying the EWA algorithm on $\mathcal{F}_{\epsilon}$, we obtain the regret bound for bounded convex losses as follows: $$\sqrt{(T/2)\log|\mathcal{F}_{\epsilon}|}+3\vch+O(1)=O\left(\sqrt{T\mathsf{VC}(\mathcal{H})\log (T/\sigma)}+\vch\right),$$
    where we have used the standard bound on the covering size $\log |\mathcal{F}_{\epsilon}|\le O(\vch\log 1/\epsilon)$~\citep{haussler1995sphere}.
    Applying the Smooth truncated Bayesian algorithm~\citep{wu2022precise} on $\mathcal{F}_{\epsilon}$ with truncation parameter $\frac{1}{T}$, we get the regret bound for Log-loss
    $$\log |\mathcal{F}_{\epsilon}|+3\vch\log T+O(1)=O(\mathsf{VC}(\mathcal{H})\log(T/\sigma)).$$
    The bound for bounded mixable loss follows by applying the Aggregating Algorithm~\cite[Chapter 3]{lugosi-book} on $\mathcal{F}_{\epsilon}$.
\end{proof}
\begin{remark}
Note that the first bound in Corollary~\ref{cor1} recovers the bound in~\citep{haghtalab2022smoothed}, while the second bound is new and  improves a $\log T$ factor for Log-loss if we use the $\vch\log T$ approximation bound of~\cite[Lemma B.2]{haghtalab2022smoothed} instead of our Proposition~\ref{them_k_tail}.
\end{remark}
\begin{corollary}
\label{cor:k_smooth}
        Let $\mathcal{H}$ be a class of finite VC-dimension and $\mathsf{S}^{\sigma}(\mu_1,\cdots,\mu_K)$ be the smooth process with multiple (known) reference measures $\mu_1,\cdots, \mu_K$. Then
        $$\Tilde{r}_T(\mathcal{H},\mathsf{S}^{\sigma}(\mu_1,\cdots,\mu_K))\le O(\sqrt{T\vch \log(KT/\sigma)})$$ under bounded convex losses,  and
        $$\Tilde{r}_T(\mathcal{H},\mathsf{P})\le O(\vch\log(KT/\sigma))$$
        under logarithmic loss and bounded mixable losses.
\end{corollary}
\begin{proof}
    This follows directly from Corollary~\ref{cor1} and Proposition~\ref{prop_equiv} and~\ref{prop_transit}.
\end{proof}

\section{Proof of Theorem~\ref{thm_sqrt_bound}}
\label{sec:proofthm8}
Before we present a formal proof of Theorem~\ref{thm_sqrt_bound}, we first develop some technical concepts that are necessary for our proof. Let $\mathcal{H}\subset \{0,1\}^{\mathcal{X}}$ be a binary valued class. For any $i<j\le M$ and $\x^M\in \mathcal{X}^M$, we define the \emph{agreed-mismatch} number of $\mathcal{H}$ on discrete interval $[i,j]:=\{i,i+1,\cdots,j\}$ to be
$$\mathsf{AM}(\mathcal{H},i,j,\x^M)=\sup_{h_1,h_2\in \mathcal{H}}\left\{\sum_{t=i}^j1\{h_1(\x_t)\not=h_2(\x_t)\}:\forall t<i,~h_1(\x_t)=h_2(\x_t)\right\}.$$
Note that the error bound $E$ in Algorithm~\ref{alg:1} at the end of each epoch is always a \emph{lower bound} for the agreed-mismatch number at that epoch (with $i,j$ being the start and end of the epoch, respectively). We have the following key lemmas for bounding the \emph{agreed-mismatch} number:

\begin{lemma}
\label{lem_exchange_bound}
    Let $\mathcal{H}\subset \{0,1\}^{\mathcal{X}}$ be a class of finite VC-dimension and $\mu$ be an arbitrary distribution over $\mathcal{X}$. Then for any $i<j\le M\in \mathbb{N}^+$, we have for all $E\ge 0$
    $$\mathrm{Pr}_{\x^M\sim \mu^{\otimes M}}\left[\mathsf{AM}(\mathcal{H},i,j,\x^M)\ge E\right]\le e^{2\vch\log j-(i\cdot E)/j}.$$
\end{lemma}
\begin{proof}
    We use a symmetric argument as in the proof of Lemma~\ref{lem_k_tail}. The event $\mathsf{AM}(\mathcal{H},i,j,\x^M)\ge E$ is equivalent to
    $$A=\left\{\exists h_1,h_2\in \mathcal{H}~s.t.~\forall t<i,~h_1(\x_t)=h_1(\x_t)\text{ and }\sum_{t=i}^j1\{h_1(\x_t)\not=h_2(\x_t)\}\ge E\right\}.$$

    By symmetries of $i.i.d.$ samples, we can fix $\x^j$ and perform a uniform random permutation $\pi$ over $[j]$. Now, for the event $A$ to happen, there must be some $h_1,h_2\in \mathcal{H}$ that differ on at least $E$ positions in $\x^j$. Denote $B\ge E$ to be the number of mismatches of $h_1,h_2$ on $\x^j$. In order for the event $A$ to happen, one must not switch any $t\in [i,j]$ for which $h_1(\x_t)\not=h_2(\x_t)$ to $[1,i-1]$ under permutation $\pi$. This happens with probability upper bounded by (using a simple combinatorial argument):
    $$\frac{\binom{j-i}{B}}{\binom{j}{B}}\le \left(1-\frac{i}{j}\right)^B\le e^{-(i\cdot B)/j}\le e^{-(i\cdot E)/j},$$
    where we have used the fact that $\frac{a}{b}\ge \frac{a-t}{b-t}$ for all $b\ge a\ge t$ and $e^{-(i\cdot B)/j}$ is decreasing on $B$.

    The lemma follows by applying a union bound on all the pairs $(h_1,h_2)$ restricted on $\x^j$ and an application of Sauers's lemma~\citep{shalev2014understanding}, and  $\mathrm{Pr}_{\x^j}[A(\x^j)]\le \sup_{\x^j}\mathrm{Pr}_{\pi}[A(\x^{\pi(j)})]$ due to symmetries of $i.i.d.$ samples.
\end{proof}

The following lemma is the key element of our proof.

\begin{lemma}
\label{lem_bad_event}
    Let $\mathcal{H}\subset \{0,1\}^{\mathcal{X}}$ be a class of finite VC-dimension and $\mu$ be an arbitrary distribution over $\mathcal{X}$. For any $E\le M\in \mathbb{N}^+$ and $\x^M\in \mathcal{X}^M$, we denote by $A$  the event that there exists 
    $$n>\frac{\log M}{\log(E/(2\vch \log M+\log (M^2/\beta)))}$$ 
    and $1=i_1<i_2<\cdots<i_{n+1}=M$ such that
    $$\forall j\le n,~\mathsf{AM}(\mathcal{H},i_j,i_{j+1},\x^M)\ge E.$$
    Then
    $$\mathrm{Pr}_{\x^M\sim \mu^{\otimes M}}[A]\le \beta.$$
\end{lemma}
\begin{proof}
    Let $B_{i,j}$ be the event that $\{\mathsf{AM}(\mathcal{H},i,j,\x^M)\ge E \text{ and }j\le (E\cdot i)/(2\vch \log M+\log (M^2/\beta))\}$. By Lemma~\ref{lem_exchange_bound}, we have for all $i,j$ and $\beta>0$
    $$\mathrm{Pr}[B_{i,j}]\le \frac{\beta}{M^2}.$$
    Using the union on all the pairs $(i,j)$, we have
    $$\mathrm{Pr}[\exists i,j,~B_{i,j}]\le \beta.$$

    Let $B=\bigcap_{i,j}\neg B_{i,j}$. Then $\mathrm{Pr}[B]\ge 1-\beta$. Note that the event $\neg B_{i,j}$ implies that if $\mathsf{AM}(\mathcal{H},i,j,\x^M)\ge E$ then 
    $$
    j\ge \frac{E\cdot i}{2\vch \log M+\log (M^2/\beta)}.
    $$  
    Conditioning on the event $B$ happening, we have, if event $A$ (defined in the statement of Lemma~\ref{lem_bad_event}) happens then
    $$\forall j\le n,~i_{j+1}\ge (E\cdot i_j)/(2\vch \log M+\log (M^2/\beta)),$$
    since event $A$ implies $\mathsf{AM}(\mathcal{H},i_{j+1},i_j,\x^M)\ge E$ for all $j\le n$.
    Note that $i_2\ge E$,  hence by induction
    $$i_{n+1}\ge \left(\frac{E}{2\vch \log M+\log (M^2/\beta)}\right)^n.$$
    However, since we also have $i_{n+1}\le M$, we must have
    $$n\le \frac{\log M}{\log(E/(2\vch \log M+\log (M^2/\beta)))}.$$
    This contradicts the definition of $A$ (the event $A$ requires number $n$ to be \emph{larger} than the above quantity) and implies that conditioning on event $B$, event $A$ cannot happen. Therefore, we have $\mathrm{Pr}[A\mid B]=0$, i.e., $\mathrm{Pr}[A\cap B]=0$. This implies
    $$\mathrm{Pr}[A]=\mathrm{Pr}[A\cap B]+\mathrm{Pr}[A\cap\neg B]\le \mathrm{Pr}[A\cap \neg B]\le \mathrm{Pr}[\neg B]\le \beta$$
    as needed.
\end{proof}

\begin{remark}
    We remark that the results in both Lemma~\ref{lem_exchange_bound} and~\ref{lem_bad_event} hold for a general exchangeable process as well. Note that these two results cannot be applied directly on the processes in $\mathsf{U}_K^1$ since they require the underlying process to be exchangeable. Our key approach, as in Proposition~\ref{prop:decouple}, is to  decouple the process in $\mathsf{U}_K^1$ into \emph{conditional} $i.i.d.$ processes.
\end{remark}

We now prove the upper bound of Theorem~\ref{thm_sqrt_bound}.

\begin{proof}[Proof of Theorem~\ref{thm_sqrt_bound} (Upper Bound)]
    Let $\pmb{\nu}^T\in \mathsf{U}_K^1$ be an arbitrary dynamic changing process with cost $K$. We denote by $X^T$ the random process generated by $\pmb{\nu}^T$. Note that the main difficulty here is to deal with the dependency among the samples in $X^T$. Our key idea is to \emph{extend} the sample $X^T$ into a coupled sample $V^{KT}$ such that the first $T$ samples in $V^{KT}$ match $X^T$ and each \emph{conditional} distribution selected for generating $X^T$ contributes \emph{exactly} $T$ samples in $V^{KT}$. We denote $V^{(k)}=V_{k_1},\cdots,V_{k_T}$ to be the samples generated by the $k$th conditional distribution (that is used to generate $X^T$), where $k\le K$. We also denote by  $X^{(k)}$  the truncated sample of $V^{(k)}$ on $V^T$. By Proposition~\ref{prop:decouple}, $V^{(k)}$ is a conditional $i.i.d.$ process, conditioning on $V^{k_1-1}$. Therefore, the unconditioned process $V^{(k)}$ is a \emph{mixture} of $i.i.d.$ processes, thus \emph{exchangeable}. Note that the truncated process $X^{(k)}$ need not be exchangeable.
    
    Taking $N=\sqrt{(T\cdot \vch\log T)/K}$ in Algorithm~\ref{alg:1}, we show that the claimed regret upper bound holds. Let $E=N/K$ and 
    $$n= \frac{\log T}{\log(E/(2\vch \log T+\log (T^2K/\beta)))}+1.
    $$ 
    We show that w.p. $\ge 1-\beta$, the parameter $s$ in Algorithm~\ref{alg:1} is upper bounded by $nK$. Suppose otherwise, we have the algorithm reenter the while loop at least $nK$ times. Denote $i_1<i_2<\cdots<i_{Kn}$ to be the time steps of reentering the while loop. Note that by construction of Algorithm~\ref{alg:1}, we exit the while loop only if the \emph{agreed-mismatch} number at current phase is larger than $N$. Therefore, we have, for each of the phases $i_{l+1}-i_l$, there must be some $k\le K$ such that $X^{(k)}$ contributes at least $N/K$ mismatches. This implies that there exists some $k\le K$ and indexes $t_1,\cdots,t_n$ (which is a sub-sequence of $i_1,\cdots,i_{nK}$) such that $X^{(k)}$ contributes at least $N/K$ mismatches in all the phases $t_{j+1}-t_j$ with $j\le n$ (note that here the phase $t_{j+1}-t_j$ may combine multiple phases of form $i_{l+1}-i_l$). Therefore, the \emph{agreed-mismatch} number restricted only on $X^{(k)}$ at each phase $t_{j+1}-t_j$ is larger than $N/K$. This is because the phase $t_{j+1}-t_j$ includes a sub-phase $i_{l+1}-i_l$ such that the \emph{agreed-mismatch} number restricted on $X^{(k)}$ for the sub-phase is larger than $N/K$. Taking $h_1,h_2$ to be the functions that whiteness such a agree-mismatch number, we have $h_1,h_2$ also agrees on $\x^{t_{j}-1}$ and differs on at least $N/K$ positions on $t_{j+1}-t_j$. Hence  the agree-mismatch number restricted on $X^{(k)}$ on phase $t_{j+1}-t_j$ is also larger than $N/K$. Since $X^{(k)}$ is a prefix of $V^{(k)}$, this implies the event of Lemma~\ref{lem_bad_event} restricted on $V^{(k)}$ happens. By Lemma~\ref{lem_bad_event} and exchangability of $V^{(k)}$, we have the event $A$ in Lemma~\ref{lem_bad_event} with the selected $n$ happens w.p. $\le \beta/K$ for each $V^{(k)}$. Using a union bound on all the $V^{(k)}$s we have the assumed event (i.e., $s>nK$) happens w.p. $\le \beta$.

    Taking $\beta=\frac{1}{T}$ and conditioning on the event $s\le nK$, we now split the regret into two parts -- one that is incurred by the mismatches and the other incurred by the adaptive EWA algorithm. Let $T_1,\cdots,T_s$ be the lengths of the the epochs. We have, by standard results~\cite[Theorem 2.3]{lugosi-book}, that the regret contributed by EWA algorithm is upper bounded by
    \begin{align*}
        \sum_{a=1}^s\sqrt{4T_a\cdot\vch\log T}&\le \sqrt{4sT\cdot\vch\log T}\\
        &\le O(\sqrt{KT\cdot\vch\log T}).
    \end{align*}
    where the first inequality follows from  Cauchy–Schwartz and  $\sum_a T_a = T$, while the second inequality follows from $s\le nK$ and $n=O(1/\epsilon)$ provided $K \ll (T^{(1-\epsilon)}/(\vch\log T))^{1/3}$. For the number of mismatches,  each epoch contributes at most $N$ mismatches and there are at most $s$ epochs, therefore the number of mismatches is upper bounded by
    $$sN\le O(\sqrt{KT\cdot\vch\log T}).$$

    Finally, the bad event $s>nK$ contributes at most $O(1)$ regret, since the loss is bounded by $1$ and the event happens with probability $\le \frac{1}{T}$. 
\end{proof}
\begin{remark}
\label{remark:adptivek}
    Note that the upper bound in Theorem~\ref{thm_sqrt_bound} can be made adaptive to $K$ (i.e., without knowing $K$) as well via a simple doubling trick. To see this, we set $K=1$ initially and run Algorithm~\ref{alg:1} as in the proof above. Once the algorithm has updated for more than $nK$ epochs, we update $K$ being $2K$ and rerun the algorithm with the new $K$. Taking $\beta=\frac{1}{T^2}$, we have by union bound (on the updates of $K$) w.p. $\ge 1-\frac{1}{T}$ there can be at most $\lceil\log K\rceil$ updates if the process is in $\mathsf{U}_K^1$. Therefore, the regret is upper bounded by
    $$\sum_{k=1}^{\lceil\log K\rceil}O\left(\sqrt{2^kT\cdot\vch\log T}\right)=O\left(\sqrt{KT\cdot\vch\log T}\right),$$
    as needed.
\end{remark}

We now prove the lower bound of Theorem~\ref{thm_sqrt_bound}.
\begin{proof}[Proof of Theorem~\ref{thm_sqrt_bound} (Lower Bound)]
    Let $\mathcal{X}=[0,1]\times \{1,2,\cdots, d\}$. We construct the following class of \emph{product threshold} functions
    $$\mathcal{H}=\{h_{\textbf{a}}(x,b)=1\{x\ge {a}_b\}:\textbf{a}\in [0,1]^d,~(x,b)\in [0,1]\times [d]\}.$$
    It is easy to see that $\vch=d$, since the set $(0.5,1),\cdots,(0.5,d)$ is shattered by $\mathcal{H}$, and any $d+1$ points must have two points with the same index in $[d]$, which cannot be shattered by $\mathcal{H}$. 
    
    We now describe a strategy for selecting $y^T$ and $\{\nu_1,\cdots,\nu_K\}$ that achieve the claimed lower bound for any prediction rule (possibly randomize) under absolute loss. Let $\tau$ be a Littlestone tree for threshold functions $\{h_a(x)=1\{x\ge a\}:a,x\in [0,1]\}$ of depth $K$, which is a $[0,1]$-valued full binary tree such that each path can be realized by a threshold function (see e.g.,~\citep{ben2009agnostic}). This must exist since threshold functions have infinite Littlestone dimension. We take $d$ \emph{copies} $\{\tau_1,\cdots,\tau_d\}$ of $\tau$ (i.e., the Littlestone forest). We select $y^T$ uniformly  from $\{0,1\}^T$ and select the $\nu_k$s in the following manner: let $I_1,\cdots,I_d$ be $d$ \emph{pointers} such that each $I_b$ points to a node in $\tau_b$ for all $b\in [d]$; initially all the $I_b$s point to the roots of $\tau_b$s, respectively. We partition the time horizon into $K$ epochs, each of length $T/K$. At the beginning of the $k$th epoch, we define the distribution
    $$\nu_k=\mathsf{Uniform}\{(\mathsf{V}(I_1),1),(\mathsf{V}(I_2),2)\cdots,(\mathsf{V}(I_d),d)\},$$
    where $\mathsf{V}(I_b)\in [0,1]$ denotes the value of the node in $\tau_b$ pointed to by index $I_b$. After the epoch $k$, we update the indices $I_b$s in the following manner: for any $b\in [d]$, if the number of $0$s is more than the number of $1$s for the labels in $y^T$ corresponding to sample $(\mathsf{V}(I_b),b)$ during epoch $k$, we move $I_b$ to its left child, and move to its right child otherwise.

    We now show that the strategy described above archives a regret lower bound $\Omega(\sqrt{KdT})$ for any prediction rule provided $\frac{T}{8Kd}\ge \log(2Kd)$. To see this, we note that by the selection of $y^T$, any prediction rule must incur $T/2$ \emph{actual} expected cumulative loss. For any $k\in [K]$ and $b\in [d]$, we denote $n_{k,b}$ to be the number of appearances of $(\mathsf{V}(I_b),b)$ during epoch $k$. We have by the \emph{multiplicative} Chernoff bound~\cite[Theorem 4.5(2)]{mitzenmacher2017probability} that
    $$\mathrm{Pr}\left[n_{k,b}\ge \frac{T}{2Kd}\right]\ge 1-e^{-T/(8Kd)}.$$
    Assuming $\frac{T}{8Kd}\ge \log(2Kd)$, then by union bound on all pairs $(k,b)$, w.p. $\ge \frac{1}{2}$, $n_{k,b}\ge \frac{T}{2Kd}$ for all $k\in [K]$ and $b\in [d]$. We now condition on that such an event happens, which is independent of $y^T$. By the Khinchine’s inequality, as in~\cite[Lemma 14]{ben2009agnostic}, the expected number of $1$s of the labels corresponding to $(\mathsf{V}(I_b),b)$ in epoch $k$ is bounded away from $\frac{n_{k,b}}{2}$ by $\sqrt{n_{k,b}/8}\ge \sqrt{T/(16Kd)}$. By our selection of $\nu_k$s, we know that there must be some $h\in \mathcal{H}$ such that the difference of the expected (over randomness of $y^T$) cumulative losses incurred by the predictor and by $h$ is lower bounded by:
    $$\sum_{k=1}^K\sum_{b=1}^d\sqrt{T/(16Kd)}\ge \sqrt{KdT/16}.$$
    This implies that there must \emph{exist} some $y^T$ such that the regret against the predictor is lower bounded by $\sqrt{KdT/16}$. Since our conditioning event on $\x^T$ happens w.p. $\ge 1/2$, the expected worst case regret is lower bounded by $\sqrt{KdT/64}$.

    Finally, to see the unconditional $\Omega(\sqrt{KT})$ lower bound, we can replicate the argument above with $b=1$ and note that $n_{k,1}=T/K$ holds always without invoking the multiplicative Chernoff bound.
\end{proof}

We now provide a supplement to Example~\ref{exm:failureofepoch} that demonstrates the failure of the epoch approach with \emph{any predefined} epochs.
\begin{example}
    \label{exm:failureofepoch1}
    Let $\mathcal{H}$, $\nu_1,\nu_2$ be as in Example~\ref{exm:failureofepoch}. Now, for any predefined epochs and number $M$, there are two cases: (i) there exists an epoch of length larger than $M$; (ii) all of the epochs have lengths less than $M$. For case (i), we can replicate the argument as in Example~\ref{exm:failureofepoch} to obtain an $\Omega(M)$ lower bound. For case (ii), we use $\nu_2$ to generate samples for all the $T$ steps. Since the EWA algorithm is deterministic for absolute loss (though it can be interpreted as a randomized algorithm for miss-classification loss),  by standard lower bounds (e.g.,~\cite[Lemma 14]{ben2009agnostic}) for any $n\in [T]$, there must be some $y^{n}$ and $h_i\in \{h_1,h_2\}$ such that the regret of EWA on $y^{n}$ against $h_i$ is lower bounded by $\Omega(\sqrt{n})$. Denote $n_1,\cdots,n_L$ to be the length of all epochs such that $n_l\le M$ for all $l\in [L]$. We claim that:
    \begin{equation}
    \label{eq:exm4}
        \sum_{l=1}^L\sqrt{n_l}\ge (T-M)/\sqrt{M}.
    \end{equation}
    This follows from the inequality $\sqrt{a+1}+\sqrt{b-1}\le \sqrt{a}+\sqrt{b}$ for $a\ge b$ (since the function $\sqrt{x}-\sqrt{x-1}$ is monotone decreasing). Therefore, one can "merge" the $n_l$s with as many components equal to $M$ as possible, yet the RHS of (\ref{eq:exm4}) does not increase. Since there are at least $(T-M)/M$ such components after the "merge",  (\ref{eq:exm4}) holds. By the above discussion, each epoch $l$ corresponding to some $y^{n_l}$ and $h_{i_l}$ with regret of EWA against $h_i$ is lower bounded by $\Omega(\sqrt{n_l})$. Therefore, there must be a subset $A\subset[L]$ corresponding to the same $h_i$ such that $\sum_{l\in A}\sqrt{n_l}\ge (T-M)/(2\sqrt{M})$. We choose the label $y^{n_l}$ at epoch $l$ for $l\in A$ and the label $h_i(\x_2)$ for all other epochs. This yields the lower bound $\Omega((T-M)/\sqrt{M})$ hence also $\Omega(\max\{M,(T-M)/\sqrt{M}\})\ge \Omega(T^{2/3})$, where the minimum is attained when $M=T^{2/3}$ leading to $\tilde{r}_T\ge \Omega(T^{2/3})$.
\end{example}

\section{Proof of Theorem~\ref{th10}}
\label{sec:proofthm10}
We start with the following technical lemma, along the same lines as~\cite[Lemma 7]{wu2022expected}.
\begin{lemma}
\label{lem:exp2hprop}
    Let $I_1,\cdots,I_T$ be random variables over $\{0,1\}^T$ such that there exists a number $C>0$ and partition $J_1,\cdots J_K\subset [T]$ of $[T]$ such that for all $k\in [K]$ and $k_t\in J_k$
    $$\mathbb{E}[I_{k_t}\mid I^{k_t-1}]\le \frac{C}{t},$$
    where $k_t$ is the $t$th element in $J_k$. Then for all $\beta>0$, we have
    $$\mathrm{Pr}\left[\sum_{t=1}^TI_t\ge 3CK\log (T/K)+7CK+\log(1/\beta)\right]\le \frac{1}{\beta}.$$
\end{lemma}
\begin{proof}
    Let $I'_t=I_t-\mathbb{E}[I_t\mid I^{t-1}]$, we have $I_t'$ form martingale differences. We now analyze the conditional variance of $I_t'$, i.e., $\sum_{t=1}^T\mathbb{E}[{I'}_t^2\mid I^{t-1}]$. We compute the variance for each partition $J_k$. For any $k_t\in J_k$, we have $|I_{k_t}'|\le 1$ w.p. $p_t$ and $|I_t'|\le p_t$ w.p. $1-p_t$, where $p_t\le \min\{\frac{C}{t},1\}$. Therefore, we have $\sum_{t=1}^{|J_k|}\mathbb{E}[{I'}_{k_t}^2\mid I^{k_t-1}]\le \sum_{t=1}^{|J_k|}p_t+p_t^2\le C\log |J_k|+3C$. Here, we have used the fact that $\sum_{t=1}^{\infty}p_t^2\le 2C$ and $\sum_{t=1}^{|J_k|}\frac{C}{t}\le C\log |J_k|+C$. The second inequality is straightforward; we prove the first inequality. We split the summation into $\sum_{t=1}^Cp_t^2+\sum_{t=C}^{\infty}p_t^2\le C+\sum_{t=C}^{\infty}\frac{C^2}{t^2}\le 2C$, where the first inequality follows by $p_t\le \min\{\frac{C}{t},1\}$. Now, the lemma follows by a simple application of the Bernstein's inequality for martingales~\cite[Lemma A.8]{lugosi-book} and noting that $\sum_{k=1}^K\log|J_{k}|\le K\log(T/K)$ since $\sum_{k=1}^K|J_k|=T$.
\end{proof}

\begin{proof}[Proof of Theorem~\ref{th10}]
Our proof exploits the symmetries of the product distributions of type $K$. At a high level, we will show that there exists an algorithm, i.e., the \emph{1-inclusion graph} algorithm~\citep{haussler1994predicting}, that achieves $O(K\log T+\log(1/\beta))$ cumulative error bound w.p. $\ge 1-\beta$ if the features $\x^T$ are sampling from a product distribution of type $K$ and the labels $y^T$ are realized by some $h\in \mathcal{H}$. Suppose this holds, then one will be able to derive the covering size bound through Lemma~\ref{lem:error2cover}.

We now establish the realizable cumulative error bound. Let $\Phi$ be the \emph{1-inclusion graph} algorithm, as in~\citep{haussler1994predicting}, and $\pmb{\nu}^T$ be an arbitrary product distribution of type $K$. We partition the index set $[T]$ into $K$ groups $J_1,\cdots, J_K$ such that for any indices $i,j$ belonging to the same group $J_k$, we have $\nu_i=\nu_j$. Note that such a partition will only be used in our analysis and it is \emph{unknown} to the algorithm $\Phi$. Denote by $\pi$ a random permutation such that the restriction of $\pi$ to any $J_k$ with $k\in [K]$ is uniform random permutation over $J_k$ and is independent for different $k$. Let $A$ be an arbitrary event over $\x^T$. We have by symmetries of the product distribution that:
$$\mathrm{Pr}_{\x^T\sim\pmb{\nu}^T}[A(\x^T)]=\mathbb{E}_{\pi}\mathrm{Pr}_{\x^T\sim\pmb{\nu}^T}[A(\x^T)]\le \sup_{\x^T}\mathrm{Pr}_{\pi}[A(\x^{\pi(T)})].$$
It is therefore sufficient to fix the features $\x^T$ and prove the cumulative error bound under permutation $\pi$. For any $h\in \mathcal{H}$, we denote $I_t^h$ to be the indicator that the event 
$$\Phi(\x^{\pi(t)},\{h(\x_{\pi(1)}),\cdots,h(\x_{\pi(t-1)})\})\not=h(\x_\pi(t)),
$$
i.e., the predictor $\Phi$ makes an error at time $t$ for the realizable sample of $h$. We claim that
$$\mathbb{E}_{\pi}[I_t^h\mid \x_{\pi(t+1)},\cdots,\x_{\pi(T)}]\le \frac{\vch}{t_{k_t}},$$
where $t_{k_t}$ is the position of $t$ in $J_{k_t}$ and $k_t\in [K]$ is the index such that $t\in J_{k_t}$. To see this, we have by~\cite[Theorem 2.3(ii)]{haussler1994predicting} that for any realization $\x^t$, there are at most $\vch$ positions $j\in [t]$ such that $\Phi(\x^t_{-j},h(\{\x^t_{-j})\})\not=h(\x_j)$, where $\x^t_{-j}$ is the sample of $\x^t$ by removing $\x_j$ and $h(\{\x^t_{-j}\})=\{h(\x_1),\cdots,h(\x_{j-1}),h(\x_{j+1}),\cdots,h(\x_t)\}$ be the labels of $h$ on $\x^t_{-j}$. Thus, there are at most $\vch$ such indices in $J_{k_t}$; by restricting $\pi$ on $\x^t$, we have $I_t^h=1$ only if such indices are switched to $\x_t$ under $\pi$, which happens w.p. $\le \frac{\vch}{t_{k_t}}$. Now, by the \emph{permutation invariance} of $1$-inclusion graph predictor, we have that $I_t^h$ is completely determined by $\x_{\pi(t)},\cdots,\x_{\pi(T)}$. Therefore, we have
$$\mathbb{E}_{\pi}[I_t^h\mid I_{t+1}^h,\cdots,I_T^h]=\mathbb{E}_{\pi}[I_t^h\mid \x_{\pi(t+1)},\cdots,\x_{\pi(T)}]\le \frac{\vch}{t_{k_t}}.$$
This implies that $I_1^h,\cdots,I_T^h$ form the \emph{reversed} sequence as in Lemma~\ref{lem:exp2hprop}. Invoking Lemma~\ref{lem:exp2hprop} with $C=\vch$, we have
$$\mathrm{Pr}_{\pi}\left[\sum_{t=1}^TI_t^h\ge O(K\cdot \vch\log(T/K)+\log(1/\beta))\right]\le \beta.$$ Since there are only $T^{\vch}$ functions restricted on any $\x^T$ by Sauer's lemma, we have by union bound 
$$\mathrm{Pr}_{\pi}\left[\sup_{h\in \mathcal{H}}\sum_{t=1}^TI_t^h\ge O(K\cdot \vch\log (T/K)+\log(T^{\vch}/\beta))\right]\le \beta.$$
The upper bound on the stochastic sequential covering number now follows by Lemma~\ref{lem:error2cover}.
\end{proof}

\section{Proof of Theorem~\ref{universal_k} and Corollary~\ref{cor:mixable}}
\label{sec:proofthm13}

For any hypothesis class $\mathcal{H}\subset \mathcal{Y}^{\mathcal{X}}$, the ERM rule is any function $\mathsf{ERM}:(\mathcal{X}\times \mathcal{Y})^*\rightarrow \mathcal{H}$ such that for all $t\ge 1$ and $(\x^t,y^t)\in (\mathcal{X}\times \mathcal{Y})^t$, we have
$$\sum_{i=1}^t1\{\erm(\x^t,y^t)[\x_i]\not=y_i\}=\inf_{h\in \mathcal{H}}\sum_{i=1}^t1\{h(\x_i)\not=y_i\}.$$
Let $\Phi:(\mathcal{X}\times \mathcal{Y})^*\rightarrow \mathcal{Y}^{\mathcal{X}}$ be a prediction rule, $h\in \mathcal{H}$ and $\x^T\in \mathcal{X}^T$, we denote the \emph{cumulative} error of $\Phi$ under the realizable sample of $h$ on $\x^T$ as (recall the definition in (\ref{eq:realizbalebound})):
$$\mathsf{err}(\Phi,h,\x^T)=\sum_{t=1}^T1\{\Phi(\x^{t-1},\{h(\x_1),\cdots,h(\x_{t-1})\})[\x_t]\not=h(\x_t)\}.$$
We begin with the following high probability \emph{cumulative} error bound for the ERM rule under realizable $i.i.d.$ sampling:
\begin{lemma}
\label{lem_erm}
    Let $\mathcal{H}\subset \{0,1\}^{\mathcal{X}}$ be any class with finite VC-dimension and $\erm$ be an arbitrary ERM rule of $\mathcal{H}$. Then for any distribution $\mu$ over $\mathcal{X}$ and $\beta>0$ we have w.p. $\ge 1-\beta$ over $\x^T\sim \mu^{\otimes T}$
    \begin{align*}
    \sup_{h\in\mathcal{H}}\mathsf{err}(\erm,h,\x^T)\le O((\vch\log^2 T+\log(1/\beta)\log T)\cdot \Delta)
    \end{align*}
    where $\Delta=\log(\vch\log T\log(1/\beta))$ and $O$ hides absolute constant independent of $\vch, T,\beta$.
\end{lemma}

Note that even though the samples $\x^T$ in Lemma~\ref{lem_erm} are $i.i.d.$, the \emph{predictions} made by $\erm$ rule are \emph{not} independent, which is the main technical difficulty in proving Lemma~\ref{lem_erm}.  To resolve this issue, we first establish the following key lemma which provides a general approach for converting a $\sup_h\mathbb{E}$ type bound to a $\mathbb{E}\sup_h$ bound. Our main proof technique is a \emph{perturbation} argument, which is the main technical contribution of this section. For any prediction rule $\Phi:(\mathcal{X}\times\{0,1\})^*\rightarrow \{0,1\}^{\mathcal{X}}$ and $I\subset [T]$, we define a perturbed function $\Phi^I$ such that for all $\x^t, y^t$ we have
$$\Phi^I(\x^t,y^t) = \Phi(\x^t,\tilde{y}^t),$$
where $\tilde{y}_t=y_t$ if $t\not\in I$ and $\tilde{y}_t=1-y_t$ if $t\in I$.

\begin{lemma}
\label{lem_infinite2finite}
    Let $\mathcal{H}\subset \{0,1\}^{\mathcal{X}}$ be a class of finite VC-dimension, $\mu$ be a distribution over $\mathcal{X}$, and $\mathcal{F}_{\epsilon}$ is an $\epsilon$-cover of $\mathcal{H}$ w.r.t. $\mu$ (see Lemma~\ref{lem_k_tail}), where $\epsilon = \frac{1}{2T^2}$. Then for any prediction rule $\Phi:(\mathcal{X}\times\{0,1\})^*\rightarrow \{0,1\}^{\mathcal{X}}$ we have for all $m,n\in \mathbb{N}^+$
    \begin{align*}
    \mathrm{Pr}_{\x^T\sim \mu^{\otimes T}}&\left[\sup_{h\in \mathcal{H}}\mathsf{err}(\Phi,h,\x^T)\ge m+3\mathsf{VC}(\mathcal{H})+n \right] \\
    &\qquad\qquad\qquad\le \mathrm{Pr}_{\x^T\sim \mu^{\otimes T}}\left[\sup_{f\in \mathcal{F}_{\epsilon}}\sup_{I\subset [T],|I|\le 3\mathsf{VC}(\mathcal{H})+n}\mathsf{err}(\Phi^I,f,\x^T)\ge m\right]+\frac{1}{T^n}.
    \end{align*}

\end{lemma}
\begin{proof}
    Let $A$ be the event that $$A=\left\{\x^T:\sup_{h\in \mathcal{H}}\inf_{f\in \mathcal{F}_{\epsilon}}\sum_{t=1}^T1\{h(\x_t)\not=f(\x_t)\}\le 3\mathsf{VC}(\mathcal{H})+n\right\}.$$
    We have by Lemma~\ref{lem_k_tail} that $\mathrm{Pr}[A]\ge 1-\frac{1}{T^n}$ (taking $M=T$ in the lemma). Conditioning on the event $A$ happening, we have for all $h\in \mathcal{H}$, there exists $f\in \mathcal{F}_{\epsilon}$ such that there are at most $3\mathsf{VC}(\mathcal{H})+n$ positions $t\in [T]$ such that $h(\x_t)\not=f(\x_t)$. Denote $I\subset [T]$ to be the set of such positions. Then  $\Phi$ and $\Phi^I$ have the same outputs on $\x^t$ with labeling of $h$ for all $t\in [T]$; meaning that
    $$\mathsf{err}(\Phi,h,\x^T)-\mathsf{err}(\Phi^I,f,\x^T)\le 3\mathsf{VC}(\mathcal{H})+n,$$
    since only the positions for which $h(\x_t)\not=f(\x_t)$ contribute $1$ to the difference of errors. This implies
    $$\sup_{h\in\mathcal{H}}\inf_{f\in \mathcal{F}_{\epsilon}}\inf_{I}\mathsf{err}(\Phi,h,\x^T)-\mathsf{err}(\Phi^I,f,\x^T)\le 3\mathsf{VC}(\mathcal{H})+n.$$
    The result follows by noting that
    $$\sup_{h\in\mathcal{H}}\inf_{f\in \mathcal{F}_{\epsilon}}\inf_{I}[\mathsf{err}(\Phi,h,\x^T)-\mathsf{err}(\Phi^I,f,\x^T)]=\sup_{h\in \mathcal{H}}\mathsf{err}(\Phi,h,\x^T)-\sup_{f\in \mathcal{F}_{\epsilon}}\sup_{I}(\Phi^I,f,\x^T),$$
    and removing the conditioning on $A$ by a union bound.
\end{proof}

Lemma~\ref{lem_infinite2finite} is interesting since it reduces an event of form $\sup_h$ with infinite $\mathcal{H}$ to an event of form $\sup_{f,I}$ with \emph{finite} $\mathcal{F}_{\epsilon}$ and $\{I\subset [T]:|I|\le 3\mathsf{VC}(\mathcal{H})+n\}$. The latter can be handled using union bounds if we are able to obtain a high probability error bound for $\Phi^{I}$ for any such $f$ and $I$. The following lemma establish such a result for ERM rule with $i.i.d.$ sampling.

\begin{lemma}
\label{lem_singleh}
    Let $\mathcal{H}\subset \{0,1\}^{\mathcal{X}}$ be a class of finite VC-dimension, $\mu$ be a distribution over $\mathcal{X}$. For any $h\in \mathcal{H}$ and $I\subset [T]$ with $|I|\le e$ for some integer $e\ge 1$, we have for all $\beta>0$
    $$\mathrm{Pr}_{\x^T\sim \mu^{\otimes T}}\left[\mathsf{err}(\mathsf{ERM}^I,h,\x^T)\ge O(\log T(\vch\log T+e+\log(1/\beta))\cdot \Delta)\right]\le \beta,$$
    where $\Delta=\log(e\vch\log T\log(1/\beta))$, $\erm$ is \emph{any} ERM rule, and $O$ hides absolute constant independent of $e,\vch,T,\beta$.
\end{lemma}
\begin{proof}
    Fix any $h\in \mathcal{H}$ and $I\subset [T]$ with $|I|\le e$. We denote by  $\mathsf{ERM}^I_t$  the function generated by $\mathsf{ERM}^I$ using samples $\x^{t},y^{t}$. Let $\mathsf{err}_t=\mathrm{Pr}_{\x\sim \mu}[\mathsf{ERM}^I_t(\x)\not=h(\x)]$. We now claim that for all $t\in [T]$ we have
    \begin{equation}
    \label{eq:prooflem31}
        \mathrm{Pr}_{\x^t\sim \mu^{\otimes t}}\left[\mathsf{err}_t\ge O\left(\frac{(\vch\log t+e+\log(1/\beta))\log(e\vch\log t\log(1/\beta))}{t}\right)\right]\le \beta.
    \end{equation}
    To see this, we use a symmetric argument. Let $S_1,S_2$ be two $i.i.d.$ samples of $\mu$ both of length $t$. For any $h_1,h_2\in \mathcal{H}$, we define distance $d(h_1,h_2)=\mathrm{Pr}_{\x\sim \mu}[h_1(\x)\not=h_2(\x)]$. We define two events
    $$A_1^h=\left\{\exists h'\in \mathcal{H},~d(h,h')\ge \epsilon\text{ 
and }\sum_{s\in S_1}1\{h'(s)\not=h(s)\}\le e\right\},$$
    and
    $$A_2^h=\left\{\exists h'\in \mathcal{H},~d(h,h')\ge \epsilon\text{ 
and }\sum_{s\in S_1}1\{h'(s)\not=h(s)\}\le e\text{ but }\sum_{s\in S_2}1\{h'(s)\not=h(s)\}\le \epsilon t/2\right\}.$$
Using the same argument as in Lemma~\ref{lem_k_tail}, we have $\mathrm{Pr}[A_1^h]\le 2\mathrm{Pr}[A_2^h]$. By symmetries of $i.i.d.$ distributions we can fix $S_1\cup S_2$ and perform a random permutation $\pi$ that switches coordinate $i$ of $S_1$ and $S_2$ w.p. $\frac{1}{2}$ and independent of different $i\in [t]$.  In order for the event $A_2^h$ to happen we cannot switch more than $e$ elements for which $\mathsf{ERM}^I_t(s)\not=h(s)$ with $s\in S_2$ to $S_1$. This happens with probability upper bounded by
$$\frac{1}{2^{\epsilon t/2}}\sum_{i=0}^e\binom{\epsilon t/2}{i}\le 2^{-\epsilon t/2 + (e+1)\log(\epsilon t/2)}.$$
Using a union bound on functions of $\mathcal{H}$ restricted on $S_1\cup S_2$, we have
$$\mathrm{Pr}[A_2^h]\le 2^{\vch\log t-\epsilon t/2+(e+1)\log(\epsilon t/2)}.$$
Taking 
$$\epsilon = c\cdot\left(\frac{(\vch\log t+e+\log(2/\beta))\log(e\vch\log t\log(2/\beta))}{t}\right)$$
one can make $\mathrm{Pr}[A_2^h]$ upper bounded by $\beta/2$ for some absolute constant $c>0$. The Claim~(\ref{eq:prooflem31}) follows by noting that $\mathsf{err}_t\ge \epsilon$ implies event $A_1^h$ happens by construction of $\mathsf{ERM}^I$.

We now upper bound the \emph{cumulative errors} of $\erm^I$. Let event
$$G_t=\left\{\mathsf{err}_t\le c\cdot\left(\frac{(\vch\log t+e+\log(4T/\beta))\log(e\vch\log t\log(4T/\beta))}{t}\right)\right\},$$and indicator
$$I_t=\{\mathsf{ERM}_{t-1}^I(\x_t)\not=h(\x_t)\text{ and }G_{t-1}\}.$$
We have $\mathrm{Pr}[G_t]\ge 1-\beta/(2T)$ for all $t\le T$. Note that $G_{t-1}$ is independent of $\x_t$, thus we have (since $I_t=1$ happens only when $G_{t-1}$ happens \emph{and} $\mathsf{ERM}_{t-1}^I(\x_t)\not=h(\x_t)$)
\begin{equation}
\label{eq:proofthm311}
    \mathbb{E}[I_t\mid I_1,\cdots,I_{t-1}]\le c\cdot\left(\frac{(\vch\log t+e+\log(4T/\beta))\log(e\vch\log t\log(4T/\beta))}{t}\right).
\end{equation}
 By Lemma~\ref{lem:exp2hprop} with $K=1$, $C$ being the numerator of Equation~(\ref{eq:proofthm311}) and upper bound $\log t$ by $\log T$, we have for sufficiently large $T$ that
$$\mathrm{Pr}\left[\sum_{t=1}^TI_t\ge 4c\cdot\log T(\vch\log T+e+\log(4T/\beta))\cdot \Delta+\log(2/\beta)\right]\le\beta/2,$$
where $\Delta=\log(e\vch\log T\log(4T/\beta))$.
Note that, the events $G=\cap_{t\in [T]}G_{t-1}$ and $\mathsf{ERM}^I_{t-1}[\x_t]\not=h(\x_t)$ together imply that $I_t=1$. Therefore, using the fact that $\mathrm{Pr}[A]\le \mathrm{Pr}[A\cap G]+\mathrm{Pr}[\neg G]\le \mathrm{Pr}[A\cap G]+\beta/2$ for any event $A$, we conclude
$$\mathrm{Pr}_{\x^T\sim \mu^{\otimes T}}\left[\mathsf{err}(\mathsf{ERM}^I,h,\x^T)\ge O(\log T(\vch\log T+e+\log(1/\beta))\cdot \Delta)\right]\le \beta.$$
This completes the proof.
\end{proof}

\begin{proof}[Proof of Lemma~\ref{lem_erm}]
    By Lemma~\ref{lem_infinite2finite}, it is sufficient to 
upper bound
\begin{equation}
\label{eq:prooflem29}
    \mathrm{Pr}_{\x^T\sim \mu^{\otimes T}}\left[\sup_{f\in \mathcal{F}_{\epsilon}}\sup_{I\subset [T],|I|\le 3\mathsf{VC}(\mathcal{H})+n}\mathsf{err}(\Phi^I,f,\x^T)\ge m\right].
\end{equation}
    We now take $n=\log(2/\beta)/\log T$ in Lemma~\ref{lem_infinite2finite}, i.e., $\frac{1}{T^n}=\beta/2$. By Lemma~\ref{lem_singleh} with $e=3\vch+n$ together with a union bound on $\mathcal{F}_{\epsilon}$ and $\{I\subset [T]:|I|\le 3\vch+n\}$ and letting
    $$m=O(\log T(\vch\log T+e+\log(2B/\beta))\cdot \Delta)$$
    where $\Delta=\log(e\vch\log T\log(2B/\beta))$ and $B=|\mathcal{F}_{\epsilon}|\cdot |\{I\subset [T]:|I|\le 3\vch+n\}|$, one can make the error probability~(\ref{eq:prooflem29}) upper bounded by $\beta/2$. We now observe that $\log|\mathcal{F}_{\epsilon}|\le O(\vch\log T)$ and $\log|\{I\subset [T]:|I|\le 3\vch+n\}|\le O(\vch\log T+\log(1/\beta))$. Putting everything together and simplifying the expression, we have w.p. $\ge 1-\beta$ over $\x^T\sim\mu^{\otimes T}$
    $$\sup_{h\in \mathcal{H}}\mathsf{err}(\erm,h,\x^T)\le O(\log T(\vch\log T+\log(1/\beta))\log(\vch\log T\log(1/\beta))).$$
    This completes the proof.
\end{proof}
The following lemma is the key element in our proof.
\begin{lemma}
\label{lem:decouple_error}
    For any random process $X^T\in \mathsf{U}_K^1$, we denote $V^{KT}$ and $V^{(k)}=V_{k_1},\cdots,V_{k_T}$ for $k\in [K]$ as in Proposition~\ref{prop:decouple}. We have for all $k\in [K]$ w.p. $\ge 1-\beta$ over $V^{KT}$
    \begin{align*}
        \sup_{h\in \mathcal{H}}\sum_{t=1}^T 1\{\erm(V^{k_{t}-1},&\{h(V_1),\cdots,h(V_{k_t-1})\})[V_{k_t}]\not=h(V_{k_t})\}\le \\
        &O((\vch\log^2 T+\log T\log(1/\beta))\log(\vch\log T\log(1/\beta))),
    \end{align*}
    where $\erm$ is any ERM rule and $O$ hides absolute constant independent of $\vch$, $T$ and $\log(1/\beta)$.
\end{lemma}
\begin{proof}
    By Proposition~\ref{prop:decouple}, we have $V^{(k)}$ is an $i.i.d.$ process conditioning on $V^{k_1-1}$. The key observation is that the ERM rule over $V^{KT}$ restricted on $V^{(k)}$ is still an (randomized) ERM rule, since we have assumed that the samples are \emph{realizable}. Conditioning on any $V^{k_1-1}$, the upper bound then follows by Lemma~\ref{lem_erm} since it only requires that the ERM rule at each time step $k_t$ is independent of $V_{k_t}$ and it does not depend on how the ERM functions are selected (even if the selections are randomized). To remove the conditioning on $V^{k_1-1}$, we use the following law of total probability: for any event $A\subset V^{KT}$ we have
    $$\mathrm{Pr}[A]=\mathbb{E}_{V_{k_1-1}}\left[\mathrm{Pr}[A\mid V^{k_1-1}]\right]\le \sup_{k_1,V_{k_1-1}}\mathrm{Pr}[A\mid V^{k_1-1}].$$
    The lemma now follows by taking $A$ to be the event in the statement of the lemma.
\end{proof}

We now ready to prove Theorem~\ref{universal_k}.

\begin{proof}[Proof of Theorem~\ref{universal_k}]
    We first observe that for any prediction rule the cumulative error on $X^T$ is less than the cumulative error on $V^{KT}$. Using Lemma~\ref{lem:decouple_error} and a union bound on all the $K$ subsequences $V^{(k)}$, we have for any ERM rule $\erm$, w.p. $\ge 1-\beta$  over $V^{KT}$, the cumulative error $$\sup_{h\in \mathcal{H}}\mathsf{err}(\erm,h,V^{KT})\le O(K(\vch\log^2 T+\log T\log(K/\beta))\cdot \Delta),$$where $\Delta=\log(\vch\log T\log(K/\beta))$. 
 Since $\mathsf{err}(\erm,h,X^T)\le \mathsf{err}(\erm,h,V^{KT})$, the sequential covering size then follows by Lemma~\ref{lem:error2cover}.
\end{proof}

Finally, we prove Corollary~\ref{cor:mixable}.
\begin{proof}[Proof of Corollary~\ref{cor:mixable}]
The upper bounds follow directly by Proposition~\ref{pro:cover2regret} and Theorem~\ref{universal_k} by taking $\beta=\frac{1}{T}$. We only need to prove the lower bound for log-loss. For the $\Omega(Kd)$ lower bound, we consider the same hard class $\mathcal{H}$ as in the lower bound proof of Theorem~\ref{thm_sqrt_bound} (in Appendix~\ref{sec:proofthm8}) and the \emph{Littlestone forests} $\tau_1,\cdots,\tau_d$ with pointers $I_b$s. We partition the time steps into $K$ epochs. At each epoch $k$, we use the same $\nu_k$ as in the lower bound proof of Theorem~\ref{thm_sqrt_bound} to generate samples. We move to the next epoch if all elements in the support of $\nu_k$ (which is a uniform distribution) have appeared at least once in the sample. We then change the pointers $I_b$ of each tree $\tau_b$ in the following manner: if the prediction made by the predictor on the first appearance of $(\mathsf{V}(I_b),b)$ is $\ge \frac{1}{2}$, we update $I_b$ to its left child, and update to right child if the prediction is $<\frac{1}{2}$. It is easy to verify that the expected regret is lower bounded by $\Omega(Kd)$, provided $Kd\ll T/\log d$ by the coupon collector problem. The lower bound for $\Omega(d\log(T/d))$ follows by standard results, see e.g.,~\cite[Theorem 24]{wu2022expected}.
\end{proof}

\section{Proof of Lemma~\ref{lem_smooth_2_select} and Theorem~\ref{thm_threhold_smooth}}
\label{sec:proofsmooth}
\begin{proof}[Proof of Lemma~\ref{lem_smooth_2_select}]
    The proof is an operational interpretation of the coupling argument as in Proposition~\ref{prop1}. Let $\mu$ be the reference measure that defines the $\sigma$-smooth process $X^T$ (with $\pmb{\nu}^T$ being the joint distribution of $X^T$). For any $m\in \mathbb{N}$, we denote $V^{mT}$ to be an $i.i.d.$ process with marginal $\mu$ and $I^{mT}$ to be an $i.i.d.$ process with marginal of uniform distribution over $[0,1]$ that is independent of $V^{mT}$. We now construct a coupling between $X^T$ and $V^{mT}, I^{mT}$. Suppose we have constructed $X^{t-1}$, we have that the conditional density $\nu_t=\nu_t(X_t\mid X^{t-1})$ is determined and we denote the density $v_t(\x)=\frac{\text{d}\nu_t}{\text{d}\mu}$. To construct $X_t$, we define the random set $S_t$ as in Proposition~\ref{prop1} in the following manner: for any $V_{m(t-1)+i}$ with $i\in [m]$, if $\sigma v_t(V_{m(t-1)+i})\ge I_{m(t-1)+i}$, we include $V_{m(t-1)+i}$ to $S_t$ (and do not include otherwise). If $S_t$ is not empty, we select the first element in $S_t$ as $X_t$, else we sample a fresh independent sample $X'_t\sim \nu_t$ and let $X_t=X'_t$. It is easy to verify that the constructed process is distributed w.r.t $\pmb{\nu}^T$. Note that the main difference with the proof of Proposition~\ref{prop1} is that we used the random variables $I^{mT}$ on the selection of $S_t$ instead of the Bernoulli($\sigma v_t(V_{m(t-1)+i})$) random variables (it is easy to check these two construction results in the same distribution of $S_t$).

    We now denote $R_t=\{I_{m(t-1)+1},\cdots,I_{mt}\}$, where $R_t$ is independent of $\{V_{m(t-1)+1},\cdots,V_{mt}\}$. The above coupling process can be expressed as $X_t=f(R_t,X'_t,\{V_{m(t-1)+1},\cdots,V_{mt}\})$, where $f$ is a \emph{deterministic} function, such that w.p. $\ge 1-Te^{-m\sigma}$ over $R^T, X'^T, V^{mT}$
    $$\forall t\in [T],~f(R_t, X'_t, \{V_{m(t-1)+1},\cdots,V_{mt}\})\in \{V_{m(t-1)+1},\cdots,V_{mt}\}.$$ Let $\tilde{f}$ be the truncated function of $f$ such that if $\forall i\in [m]$,~$\sigma v_t(V_{m(t-1)+i})<I_{m(t-1)+i}$ we set $$\tilde{f}(R_t, \{V_{m(t-1)+1},\cdots,V_{mt}\})=V_{m(t-1)+1}$$ and set $\tilde{f}(R_t, \{V_{m(t-1)+1},\cdots,V_{mt}\})=f(R_t, \{V_{m(t-1)+1},\cdots,V_{mt}\})$ otherwise. We write 
    $$\tilde{X}_t=\tilde{f}(R_t, \{V_{m(t-1)+1},\cdots,V_{mt}\}).$$
    It is easy to see that w.p. $\ge 1-Te^{-m\sigma}$ over the joint distribution $(X^T,\tilde{X}^T)$ that $\forall t\in [T]$, $X_t=\tilde{X}_t$. We now observe that conditioning on $R^T$, $\tilde{X}^T$ is an adversary $m$-selection process (since $I^{mT}$ is independent of $V^{mT}$ and $\tilde{X}^T$ is independent of $X'^T$). Therefore, we have by conditioning on $R^T$ that
    $$\mathrm{Pr}\left[\Tilde{X}^T\in A\right]=\mathbb{E}\left[\mathrm{Pr}\left[\tilde{X}^T\in A\mid R^T\right]\right]\ge 1-\beta.$$
    Using a union bound we have
    $$\mathrm{Pr}\left[X^T\in A\right]\ge 1-\beta-Te^{-m\sigma}.$$
    Taking $m=K$ and the assumption that $K\ge \frac{\log(T/\beta)}{\sigma}$, one finishes the proof.
\end{proof}

\begin{proof}[Proof of Theorem~\ref{thm_threhold_smooth}]
    Let $\tilde{X}^T$ be an adversary $K$-selection process with reference measure $\mu$ over $[0,1]$. We assume that for any $x\in [0,1]$, $\mu(\{x\})=0$. This assumption can be eliminated with a more tedious argument. However, we make the assumption here for clarity of presentation.
    
    We consider the following \emph{random} partitions of interval $[0,1]$. Initially the partition $\mathcal{I}_0$ consists of only the interval $[0,1]$. At each time step $t$, we denote $\mathcal{I}_{t-1}$ to be the current partition. Let $J_t\in \mathcal{I}_t$ be the interval for which $\tilde{X}_t\in J_t$, we split $J_t$ into two parts with values $< \tilde{X}_t$ and $\ge \tilde{X}_t$ respectively (if $\tilde{X}_t$ is the end point of $J_t$ we do not split and remain on the same $J_t$). We then replace  $J_t$ with the newly split intervals in $\mathcal{I}_t$ to form the partition $\mathcal{I}_{t+1}$. Note that, one may view this partitioning process as expanding a binary tree with each node labeled by the intervals in $\mathcal{I}_t$ and expanding a leaf when the corresponding interval is split into two parts. Such a tree can be viewed as the (compressed) \emph{realization} tree in~\cite[Theorem 13]{wu2022expected} if we view the $\mathcal{I}_t$ as subsets of $\mathcal{H}$. Our goal is to bound the \emph{maximum} depth of the tree.

    For any time step $t$, we denote $J_t=[a_t,b_t]$ to be the interval for which $\tilde{X}_t\in J_t$ and
    $$\lambda_t=\frac{\max\{\mu([a_t,X_t]),\mu([X_t,b_t])\}}{\mu([a_t,b_t])}$$
    to be the \emph{splitting ratio} of $J_t$. We claim that for any $\alpha>0$, 
    \begin{equation}
        \label{eq:proofthm183}
        \mathrm{Pr}\left[\lambda_t\ge 1-\alpha\mid \Tilde{X}^{t-1}\right]\le 2\alpha K.
    \end{equation}
     To see this, we denote $\mathcal{I}_{t-1}=\{J^1,\cdots,J^{n_t}\}$ to be the partition at time $t$ before receiving $\Tilde{X}_t$, where $J_t\in \mathcal{I}_{t-1}$ and $n_t\le t$. For any interval $J^i=[a_i,b_i]\in \mathcal{I}_{t-1}$, we define the $\alpha$-margin of $J^i$ w.r.t. $\mu$ to be the intervals $[a_i,c_i]$ and $[d_i,b_i]$ such that:
     \begin{align*}
     c_i&=\sup\{x\in [a_i,b_i]:\mu([a_i,x])\le \alpha\mu([a_i,b_i])\} \\
     d_i&=\inf\{x\in [a_i,b_i]:\mu([x,b_i])\le \alpha\mu(a_i,b_i)\}.
     \end{align*}
     Let $V_1^t,\cdots,V_K^t$ be the $i.i.d.$ samples of $\mu$ that is used to generate $\tilde{X}_t$ and $B_t(\alpha)$ be the event that there exists some $V_k^t$ and $J^i\in \mathcal{I}_{t-1}$ such that $V_k^t$ is in the $\alpha$-margin of interval $J^i$. Note that for any given $V_k^t$, the probability that $V_k^t$ is in the $\alpha$-margin of some interval in $\mathcal{I}_{t-1}$ is upper bounded by $2\alpha$. We have by independence of $V_k^t$s that
    $$\mathrm{Pr}[B_t(\alpha)]\le 1-(1-2\alpha)^K\le 2\alpha K.$$ By definition of adversary $K$-selection, we have the conditional event $\{\lambda_t\ge 1-\alpha\mid \Tilde{X}^{t-1}\}$ implying that the event $B_t(\alpha)$ happens, i.e., the Equation~(\ref{eq:proofthm183}) follows.

    Let $I_t = 1\{\lambda_t\ge 1-\alpha\}$. Then  $\mathbb{E}[I_t\mid I^{t-1}]\le 2\alpha K$ and $I'_t=I_t-\mathbb{E}[I_t\mid I^{t-1}]$ form martingale differences. Using Azuma inequality~\cite[Lemma A.7]{lugosi-book}  for all $\alpha>0$ 
    \begin{equation}
    \label{eq:proofthm18}
        \mathrm{Pr}\left[\sum_{t=1}^TI_t\ge 2\alpha KT+x\right]\le \mathrm{Pr}\left[\sum_{t=1}^TI_t'\ge x\right]\le e^{-2x^2/T}.
    \end{equation}
    Taking $x\ge \sqrt{T\log(2T/\beta)}$, one can make the above probability less than $\beta/(2T)$. This implies that for any $n\le T$ and $\alpha=\frac{n-\sqrt{T\log(2T/\beta)}}{4 KT}$, w.p. $\ge 1-\beta/(2T)$, for any $\lambda_{t_1},\cdots,\lambda_{t_n}$, we have
    \begin{equation}
        \label{eq:proofthm182}
        \sum_{i=1}^n(1-\lambda_{t_i})\ge \left(n-\left(2\alpha KT+\sqrt{T\log(2T/\beta)}\right)\right)\alpha\ge \frac{\left(n-\sqrt{T\log(2T/\beta)}\right)^2}{8 KT},
    \end{equation}
    where the first inequality follows by the fact that $I_t=1$ implies $1-\lambda_t\le \alpha$. Using a union bound on all $n\le T$, we have w.p. $\ge 1-\beta/2$ that for any $n\le T$ and $\lambda_{t_1},\cdots,\lambda_{t_n}$, we have:
    \begin{equation}
        \label{eq:proofthm184}
        \sum_{i=1}^n(1-\lambda_{t_i})\ge\frac{\left(n-\sqrt{T\log(2T/\beta)}\right)^2}{8 KT}.
    \end{equation}
    
    We now claim that w.p. $\ge 1-\beta/2$, for any interval $J_t$  either $\mu(J_t)\ge \frac{\beta}{2KT^2}$ or $J_t$ is in the final partition. To see this, we note that for any interval $J_t$ at time step $t$, once $\mu(J_t)\le \frac{\beta}{2KT^2}$, the probability it will be split at any following time step is upper bounded by (using the same argument for bounding the event $B_t(\alpha)$)
    $$T\left(1-\left(1-\frac{\beta}{2KT^2}\right)^K\right)\le \frac{\beta}{2T}.$$
    Using a union bound on all the $T$ intervals, w.p. $\ge 1-\beta/2$, all $J_t$s will either satisfy $\mu(J_t)\ge \frac{\beta}{2KT^2}$ or that $J_t$ is in the final partition. By union bound, w.p. $\ge 1-\beta$, this happens simultaneously with the event of Equation~(\ref{eq:proofthm184}). Conditioning on such a joint event, suppose now there exists a decreasing chain $J_{t_1}\supsetneq J_{t_2}\cdots\supsetneq J_{t_n}$, hence
    $$\mu(J_{t_n})\le \prod_{i=1}^n\lambda_{t_i}\le e^{-\sum_{i=1}^n(1-\lambda_{t_i})}.$$
 This implies that if 
 $$n>\sqrt{8KT\log(2KT^2/\beta)}+\sqrt{T\log(2T/\beta)},$$ 
 then 
 $\mu(J_{t_n})< \frac{\beta}{2KT^2}$ and therefore the chain must terminate. 

    Combining all of the above results, we conclude  w.p. $\ge 1-\beta$ that there is no decreasing chain of length greater than 
    $$\sqrt{8KT\log(2KT^2/\beta)}+\sqrt{T\log(2T/\beta)}+1$$ 
    i.e., the \emph{realization} tree has maximum depth upper bounded by $O(\sqrt{KT\log(2KT^2/\beta)})$. The bound on the stochastic sequential covering now follows by the same argument as in~\cite[Theorem 13]{wu2022expected}. 
    
    For the reader's convenience, we outline the argument in the following discussion. We construct a \emph{sequential} function set $\mathcal{G}$ with \emph{fixed} index set $\mathcal{W}$ of size $|\mathcal{W}|=2^{\lceil \sqrt{15KT\log(2KT^2/\beta)}\rceil}$, i.e., for each $w\in \mathcal{W}$, we construct a \emph{sequential} function $g_w:\mathcal{X}^*\rightarrow \{0,1\}$. To do so, we maintain for each node in the realization tree a \emph{subset} of $\mathcal{W}$. We initially associate $\mathcal{W}$ to the root. At each time step after receiving $\Tilde{X}_t$, for each node $v$ in the realization tree, if $v$ splits at the current step, we split the associated subset $\mathcal{W}_v\subset \mathcal{W}$ into two disjoint subsets of \emph{equal} size and associate them to the newly split nodes, respectively. For any $w\in \mathcal{W}_v$, we assign the value $g_w(\Tilde{X}^t)=0$ if $w$ is in the subset associated to the new left child and $g_w(\Tilde{X}^t)=1$ otherwise. If the node $v$ does not split, we assign $g_w(\Tilde{X}^t)$ to be the value on the agreed label (of the subset of $\mathcal{H}$ associate to $v$, see construction of realization tree at the beginning of the proof) on $\Tilde{X}_t$. The process is said to have failed, if at some step a node $v$ splits but the associated set $|\mathcal{W}_v|\le 1$. Clearly, if the process does not fail until time $T$, the constructed set $\mathcal{G}$ sequentially covers $\mathcal{H}$ on $\Tilde{X}^T$. Now, the key observation is that, from the discussion above, w.p. $\ge 1-\beta$ on $\Tilde{X}^T$, any node is constructed after at most $\sqrt{8KT\log(2KT^2/\beta)}+\sqrt{T\log(2T/\beta)}+1\le \lceil \sqrt{15KT\log(2KT^2/\beta)}\rceil$ splits. Since any split will decrease the associated subset of $\mathcal{W}$ by exactly $\frac{1}{2}$, we know that the process does not fail w.p. $\ge 1-\beta$ since $|\mathcal{W}|=2^{\lceil \sqrt{15KT\log(2KT^2/\beta)}\rceil}$. Therefore, the constructed set $\mathcal{G}$ stochastic sequential covers $\mathcal{H}$ at scale $0$ and confidence $\beta$ by Definition~\ref{def-gcover}.
\end{proof}

\section{Proof of Proposition~\ref{pro:cover2regret} and Lemma~\ref{lem:error2cover}}
\label{sec:proofpro9lem10}
We prove Proposition~\ref{pro:cover2regret} and Lemma~\ref{lem:error2cover} in this appendix. These results were already proved in~\citep{wu2022expected}; however, we reproduce the proof here for completeness. We only prove the binary valued case as needed in this paper and refer to the original paper for the full real valued case.

\begin{proof}[Proof of Proposition~\ref{pro:cover2regret}]
    Let $\ell:\hat{\mathcal{Y}}\times \mathcal{Y}\rightarrow \mathbb{R}$ be a convex loss function on the first argument and bounded by $1$, $\mathcal{H}\subset \{0,1\}^{\mathcal{X}}$ be a binary valued function class and $\mathcal{G}\subset \{0,1\}^{\mathcal{X}^*}$ be a stochastic sequential cover (see Definition~\ref{def-gcover}) of $\mathcal{H}$ w.r.t. process class $\mathsf{P}$ at scale $\alpha=0$ and $\beta=\frac{1}{T}$. Let $\phi_t$ be the prediction given by the EWA algorithm over $\mathcal{G}$ at time step $t$. For any $\x^T\in \mathcal{X}^T$, we denote by
    $$R_T(\mathcal{H},\x^T)=\sup_{y^T}\sum_{t=1}^T\ell(\phi_t(\x^t,y^{t-1}),y_t)-\inf_{h\in \mathcal{H}}\sum_{t=1}^T\ell(h(\x_t),y_t).$$
    Let $A$ be the event over $\x^T$ such that for any $h\in \mathcal{H}$ there exists $g\in \mathcal{G}$ we have $\forall t\in [T]$, $h(\x_t)=g(\x^t)$. By the definition of stochastic sequential covering, we have for all $\pmb{\nu}^T\in \mathsf{P}$, $\mathrm{Pr}_{\x^T\sim \pmb{\nu}^T}[A]\ge 1-\frac{1}{T}$. We now observe that
    \begin{align*}
        \Tilde{r}_T(\mathcal{H},\mathsf{P})&\le \sup_{\pmb{\nu}^T\in \mathsf{P}}\mathbb{E}_{\x^T\sim\pmb{\nu}^T}[1\{\x^T\in A\}R_T(\mathcal{H},\x^T)]+\mathbb{E}_{\x^T\sim\pmb{\nu}^T}[1\{\x^T\not\in A\}R_T(\mathcal{H},\x^T)]\\
        &\overset{(a)}{\le}\sup_{\pmb{\nu}^T\in \mathsf{P}}\mathbb{E}_{\x^T\sim \pmb{\nu}^T}[1\{\x^T\in A\}R_T(\mathcal{H},\x^T)] + 1\\
        &\overset{(b)}{\le} \mathrm{Pr}[A]\sqrt{T/2\log|\mathcal{G}|}+1\le O(\sqrt{T\log|\mathcal{G}|})
    \end{align*}
    where $(a)$ follows by the fact that $R_T(\mathcal{H},\x^T)$ is upper bounded by $T$ and $\mathbb{E}[1\{\x^T\not\in A\}]\le \frac{1}{T}$; $(b)$ follows by the fact that conditioning on  event $A$, $\mathcal{G}$ \emph{sequentially covers} $\mathcal{H}$ (as in~\cite{ben2009agnostic}) and therefore the regret bound follows by standard result as in~\cite[Theorem 2.2]{lugosi-book}.

    The proof of upper bound for bounded mixable losses follows similar path as above, by replacing the EWA algorithm with the Aggregation Algorithm (AA) as in~\cite[Chapter 3.5]{lugosi-book} and applying the regret bound in~\cite[Proposition 3.2]{lugosi-book}. The proof for logarithmic loss needs additional treatment since log-loss is unbounded. This can be handled by the Smooth truncated Bayesian Algorithm introduced recently in~\citep{wu2022precise}, and running the algorithm over $\mathcal{G}\cup \{u\}$ with $u$ being the constant function mapping to $\frac{1}{2}$ and with truncation parameter $\alpha$.
\end{proof}
\begin{proof}[Proof of Lemma~\ref{lem:error2cover}]
    For any $I\subset [T]$ with $|I|\le B(T,\beta)$, we \emph{recursively} define the following \emph{sequential function} $g_I$. Let $\Phi$ be a prediction rule that satisfies (\ref{eq:realizbalebound}) for $\mathcal{H}$ and $\mathsf{P}$. For any $t\le [T]$ and $\x^t\in \mathcal{X}^*$, we define
    $$g_I(\x^t)=\begin{cases}
    \Phi(\x^t,g_I(\x^1),\cdots,g_I(\x^{t-1})),\text{ if }t\not\in I,\\
    1-\Phi(\x^t,g_I(\x^1),\cdots,g_I(\x^{t-1})),\text{ if }t\in I
    \end{cases},
    $$where $g_I(\x^0)$ is understood as empty (which is not required by definition of prediction rule). Now, for any $\pmb{\nu}^T\in\mathsf{P}$, we have by (\ref{eq:realizbalebound}) that w.p. $\ge 1-\beta$ over $\x^T\sim \pmb{\nu}^T$, $\Phi$ makes at most $B(T,\beta)$ cumulative errors for all $h\in \mathcal{H}$. Taking any $\x^T$ in such event and $h\in \mathcal{H}$, we have $g_I(\x^t)$ sequentially covers $h$ by our construction above if $I\subset [T]$ is the positions for which $h(\x_t)\not=\Phi(\x^t,h(\x_1),\cdots,h(\x_{t-1}))$ where $|I|\le B(T,\beta)$. Clearly, the class $\mathcal{G}$ consisting of all such functions $g_I$ is the desired stochastic sequential cover of $\mathcal{H}$ w.r.t. $\mathsf{P}$ at scale $\alpha=0$ and confidence $\beta$. The upper bound on $|\mathcal{G}|$ follows easily by counting the number of $I$s, see~\cite[Lemma 12]{ben2009agnostic}.
\end{proof}

\section{Real valued function class via embedding}
\label{sec:real}
We briefly discuss how our results can be extended to real valued functions via embedding of $\mathcal{H}$ through stochastic sequential covering. This will be mostly interesting for the log-loss, see e.g.~\citep{bhatt2021sequential}. To do so, we consider the class as in~\citep{bhatt2021sequential}
$$\mathcal{F}=\{\theta_1 h+\theta_2 (1-h):\theta_1,\theta_2\in [0,1]\text{ and }h\in \mathcal{H}\},$$
where $\mathcal{H}\subset\{0,1\}^{\mathcal{X}}$ is a class of finite VC-dimension. Suppose now we have a stochastic sequential covering set $\mathcal{G}$ of $\mathcal{H}$ w.r.t. some random process class $\mathsf{P}$ at scale $0$ and confidence $\beta=\frac{1}{T}$. We can then choose a minimal discretization $J\subset [0,1]$ such that for any $a\in [0,1]$ there exists $b\in J$ such that $|a-b|\le \frac{1}{T}$. Clearly, we have $|J|\le T$.  Now, for any $a_1,a_2\in J$, we construct the class $\mathcal{F}_{a_1,a_2}=\{a_1 h+a_2(1-h):h\in \mathcal{H}\}$. We have $\mathcal{F}'=\bigcup_{a_1,a_2\in J}\mathcal{F}_{a_1,a_2}$ \emph{uniformly} $\frac{1}{T}$-covers $\mathcal{F}$ (i.e., for all $\x\in\mathcal{X}$ and $f\in \mathcal{F}$ there exists $f'\in \mathcal{F}'$ such that $|f(\x)-f'(\x)|\le \frac{1}{T}$) and all of the $\mathcal{F}_{a_1,a_2}$ are \emph{isomorphic} to $\mathcal{H}$. Therefore, we can construct a sequential $0$-covering set $\mathcal{G}_{a_1,a_2}$ for $\mathcal{F}_{a_1,a_2}$ using the sequential $0$-covering set $\mathcal{G}$ for $\mathcal{H}$ by setting $g_{a_1,a_2}=a_1g+a_2(1-g)$. Let $\mathcal{G}'=\bigcup_{a_1,a_2\in J}\mathcal{G}_{a_1,a_2}$, we have $\mathcal{G}'$ sequentially $\frac{1}{T}$-covers $\mathcal{F}$ w.r.t. $\mathsf{P}$ at confidence $\frac{1}{T}$. Since there are at most $T^2$ such pairs $(a_1,a_2)$, we have $|\mathcal{G}'|\le T^2|\mathcal{G}|$. Using~\cite[Theorem 4]{wu2022expected}, we arrive at the following expected worst case regret bound under log-loss
$$\Tilde{r}_T(\mathcal{F},\mathsf{P})\le \log|\mathcal{G}|+2\log T+O(1).$$
Specializing to the class $\mathsf{U}_K^1$, and using Theorem~\ref{universal_k}, we arrive at the following bound under log-loss
$$\Tilde{r}_T(\mathcal{F},\mathsf{U}_K^1)\le O(K\cdot \vch \log^3 T\log(\vch\log(KT))).$$
We can also construct more complicated classes $\mathcal{F}$ in a similar fashion as above. More generally we can study the case when $\mathcal{F}$ has bounded scale sensitive VC-dimension (i.e., the fat-shattering dimension); however, this is out of the scope of this paper and we refer to the discussions in~\citep{block2022smoothed} for the smooth adversary processes with \emph{known} reference measure and in~\citep{wu2022expected} for the universal (unknown) $i.i.d.$ processes.

\end{document}